\documentclass[notitlepage]{article}

\usepackage{arxiv}

\usepackage{amsthm}
\usepackage{amsfonts}       
\usepackage{amssymb}
\usepackage{booktabs}       
\usepackage{bm}
\usepackage{enumitem}
\usepackage[T1]{fontenc}    
\usepackage{hyperref}       
\usepackage[utf8]{inputenc} 
\usepackage{url}            
\usepackage{makecell}
\usepackage{mathtools}
\usepackage{multirow}
\usepackage{nicefrac}       
\usepackage{microtype}      
\usepackage[textsize=tiny]{todonotes}
\usepackage{thmtools}
\usepackage{xcolor}         
\usepackage[ruled,vlined,linesnumbered]{algorithm2e}

\hypersetup{
  colorlinks=true,
  linkcolor=blue,
  citecolor=blue
}

\SetKwInput{KwInput}{Input}
\SetKwInput{KwInit}{Initialize}
\SetKwInput{KwTest}{Test}
\SetKwInput{KwDef}{Definition}
\SetKwInput{KwReturn}{Return}

\usepackage{makecell}
\usepackage{multirow}

\usepackage{tikz}
\usetikzlibrary{tikzmark}
\usetikzlibrary{fit,calc}
\usetikzlibrary{arrows.meta,positioning}

\usepackage[
  backend=biber,
  style=numeric-comp,
  maxbibnames=9,
  maxcitenames=1,
  uniquelist=false,
  uniquename=false,
  isbn=false,
  sorting=nty,
  doi=false
]{biblatex}
\makeatletter
\patchcmd{\@algocf@start}
  {-1.5em}
  {0pt}
  {}{}
\makeatother

\newtheorem{theorem}{Theorem}
\newtheorem{lemma}[theorem]{Lemma}
\newtheorem{proposition}[theorem]{Proposition}

\newtheorem{corollary[theorem]}{Corollary}
\newtheorem{definition}[theorem]{Definition}

\newtheorem{remark}{Remark}
\newtheorem{assumption}{Assumption}

\newtheorem*{assumption*}{Assumption}

\let\oldremark\remark
\renewcommand{\remark}{\oldremark\normalfont}

\DeclareMathOperator*{\argmax}{argmax}
\DeclareMathOperator*{\argmin}{argmin}

\newcommand{\red}[1]{\textcolor{red}{#1}}
\definecolor{kh}{HTML}{FFA500}
\definecolor{yz}{HTML}{6495ED}
\definecolor{yellow}{HTML}{F3E5AB}

\newcount\Comments  
\newcommand{\kibitz}[2]{\ifnum\Comments=1\textcolor{#1}{#2}\fi}

\title{Offline Constrained Reinforcement Learning \\ under Partial Data Coverage}

\author{%
  Seokmin Ko \\
  KAIST\\
  \texttt{komin0407@kaist.ac.kr} \\
  \And
  Ambuj Tewari \\
  University of Michigan\\
  \texttt{tewaria@umich.edu} \\
  \And
  Kihyuk Hong \\
  KAIST\\
  \texttt{kihyukh@kaist.ac.kr}
}

\begin{document}

\maketitle

\begin{abstract}
We study offline constrained reinforcement learning with general function approximation in discounted constrained Markov decision processes.
Prior methods either require full data coverage for evaluating intermediate policies, lack oracle efficiency, or requires the knowledge of data-generating distribution for policy extraction.
We propose PDOCRL, an oracle-efficient primal-dual algorithm based on a decomposed linear-programming formulation that makes the policy an explicit optimization variable.
This avoids policy extraction that requires the knowledge of data-generating distribution, and only uses standard policy-optimization, online linear-optimization, and linear-minimization oracles.
We show that saddle-point formulations using general function approximation can have spurious saddle points even when an optimal solution is realizable, and identify a stronger realizability condition under which every restricted saddle point is optimal.
Under this condition and partial coverage of an optimal policy, PDOCRL returns a near-optimal, near-feasible policy with a \(\widetilde{\mathcal O}(\epsilon^{-2})\) sample guarantee, without access to the data-generating distribution.
Empirically, PDOCRL is competitive with strong baselines on standard offline constrained RL benchmarks.
\end{abstract}

\section{Introduction}

Offline constrained reinforcement learning aims to learn, from a fixed dataset, a policy that achieves high reward while satisfying safety or resource constraints.
We study this problem in the framework of discounted constrained Markov decision processes (CMDPs)\parencite{altman1999constrained}, with one primary reward and several auxiliary signals whose discounted returns must exceed prescribed thresholds.
This formulation is natural when online exploration is risky or expensive but logged data are available, as in autonomous driving, healthcare, and robotics \parencite{levine2020offline,jiang2024offline}.

The main obstacle in offline CMDPs is data coverage: reliable optimization is only possible for policies whose state-action occupancy is sufficiently supported by the dataset.
In unconstrained offline RL, recent pessimistic actor-critic methods show that one can still obtain meaningful guarantees under insufficient coverage by updating policies against pessimistic critics \parencite{cheng2022adversarially,zhu2023importance}.

In the constrained setting, however, applying the pessimistic critic idea within a primal-dual framework introduces an additional obstruction. A primal-dual method must update dual multipliers based on estimates of constraint violations, which requires evaluating the constraint values of the current policy. As discussed by \textcite{hong2024primal}, this entails accurately estimating constraint values for \textit{intermediate policies} encountered during optimization. Under partial data coverage, these intermediate policies may not be supported by the dataset, making the dual updates unreliable and effectively reintroducing strong coverage requirements.

This raises the following question:
\begin{quote}
\textit{Can we design an offline CMDP algorithm that avoids evaluating unsupported intermediate policies, while remaining provably sample efficient under partial data coverage and computationally efficient with access to standard oracles?}
\end{quote}

Here and throughout, we use \emph{oracle efficient} in the sense that the algorithm performs only a polynomial number of elementary operations and calls to basic optimization oracles for the function classes.
Specifically, no-regret online linear optimization, no-regret policy optimization, and linear minimization over given function classes, rather than requiring a generic, intractable optimization such as general min-max optimization over these classes.

In this paper, we answer this question affirmatively, but not via an actor-critic method.
Instead, we adopt the linear-programming (LP) view of CMDPs, which optimizes over occupancy measures or density ratios rather than repeatedly evaluating unsupported intermediate policies.
This viewpoint is well matched to partial coverage, but under general function approximation it exposes two distinct challenges.
First, restricted saddle-point formulations can admit \emph{spurious saddle points}: even when the function classes contain an optimal feasible policy, its occupancy measure, and the optimal action-value functions for the relevant Lagrangian reward combinations, a saddle point need not correspond to an optimal policy.
Second, existing LP-based methods typically recover a policy from a learned occupancy measure or density ratio using the data distribution $\mu_D$ \parencite{zhan2022offline,rashidinejad2022optimal,lee2021optidice,lee2022coptidice,zhang2024safe}, which is problematic in practice.

We address these two issues separately.
For the first, we show that a naive restricted saddle formulation is indeed insufficient: spurious saddle points can persist even when the class realizes the optimal solution and the optimal action-value functions for all Lagrangian reward combinations.
We then identify a condition under which this pathology disappears.

The policy-extraction issue is handled separately by a decomposed LP formulation and a reparameterization that make the policy an explicit optimization variable.
This lets us optimize policies directly, rather than first learning an occupancy measure and then extracting a policy using the unknown data distribution $\mu_D$.
The resulting method, PDOCRL, is an oracle-efficient primal-dual algorithm for the decomposed empirical saddle problem.
Under our realizability and coverage assumptions, a near saddle point found by PDOCRL yields a near-optimal and near-feasible policy with a $\widetilde{\mathcal O}(\epsilon^{-2})$ sample complexity guarantee under partial coverage.
As summarized in Table~\ref{table:comparison}, this separates PDOCRL from prior offline constrained RL methods: PDCA \parencite{hong2024primal} and WSAC \parencite{wei2024adversarially} require solving intractable min-max problems, while POCC \parencite{zhang2024safe} requires $\mu_D$ and access to another auxiliary function class.
Thus PDOCRL is the only method in our comparison that simultaneously provides partial-coverage guarantees, oracle efficiency, and no dependence on $\mu_D$ or a reference policy.

\begin{table}[t]
\caption{Algorithms for offline constrained RL. Red entries indicate an undesirable property}
\label{table:comparison}
\centering
\begin{tabular}{cccccc}
 \toprule
 Algorithm & \makecell{Partial \\ coverage} & \makecell{Oracle \\ efficient} & \makecell{Requires \\ $\mu_D$ or $\pi_{\mathrm{ref}}$} & \makecell{Function \\ Approximation} \\
 \midrule
 PDCA \parencite{hong2024primal} & \red{No} & \red{No} & No & $\forall \pi: Q^\pi \in \mathcal{Q}, w^\pi \in \mathcal{W}$ \\
 WSAC \parencite{wei2024adversarially} & Yes & \red{No} & \red{Yes} & $w^{\pi_{\mathrm{ref}}} \in \mathcal{W}$, $\forall \pi: Q^\pi \in \mathcal{Q}$ \\
 MBCL \parencite{le2019batch} & \red{No} & Yes & No & $\forall \pi, Q: \mathcal{T}^\pi Q \in \mathcal{Q}$ \\
 POCC \parencite{zhang2024safe} & Yes & Yes & \red{Yes} & $w^\ast \in \mathcal{W}$, $\forall w \in \mathcal{W}: x_w \in \mathcal{X}$  \\
 \textbf{PDOCRL (Ours)} & Yes & Yes & No & $w^\ast \in \mathcal{W}$, $\forall \pi: Q^\pi \in \mathcal{Q}$ \\
 \bottomrule
\end{tabular}
\end{table}

We also evaluate PDOCRL on deep offline safe RL benchmarks.
On BulletGym tasks with DSRL datasets \parencite{gronauer2022bullet,liu2023datasetsbenchmarksofflinesafe}, it is competitive with strong baselines such as BCQ-Lag, BEAR-Lag, CPQ, CDT, and COptiDICE \parencite{xu2022constraints,liu2022constrained,lee2022coptidice}.

\subsection{Related work}

\paragraph{Offline constrained RL with function approximation}
Several recent works study offline constrained RL with function approximation
through policy-value, actor-critic, or primal-dual formulations
\parencite{le2019batch,hong2024primal,wei2024adversarially}.
Batch constrained policy learning and primal-dual-critic methods require
estimating the reward and constraint values of intermediate policies, which
typically leads to full-coverage or similarly strong coverage assumptions
\parencite{le2019batch,hong2024primal}.
The weighted actor-critic method of \textcite{wei2024adversarially} handles
partial coverage by using importance-weighted critics, but its formulation
depends on a reference policy or the data distribution and is not oracle
efficient.

\paragraph{Density-ratio and DICE methods}
A closely related line of work uses marginalized importance sampling to rewrite target-policy
or occupancy-measure objectives as expectations under the offline data
distribution.
This idea appears in off-policy evaluation and DICE-style
distribution-correction methods
\parencite{liu2018breaking,xie2019towards,nachum2019dualdice,
zhang2020gendice,zhang2020gradientdice,uehara2020minimax,
yang2020regularized,nachum2020fenchel}.
It has also been adapted to offline policy optimization through methods such as
AlgaeDICE and OptiDICE
\parencite{nachum2019algaedice,lee2021optidice}.
These methods motivate our density-ratio parameterization
\(\mu = w\mu_D\), but they do not by themselves resolve the constrained
offline-RL issues of oracle-efficient optimization and policy
recovery without access to the unknown data distribution.

\paragraph{LP-based offline RL under partial coverage}
On the theory side, there are LP- and density-ratio-based analyses for offline RL under partial coverage \parencite{zhan2022offline,chen2022offline,rashidinejad2022optimal,
ozdaglar2023revisiting,zhu2023importance}.
These methods are naturally aligned with partial coverage because they reason about covered occupancies rather than requiring accurate evaluation of arbitrary intermediate policies.
However, relative to PDOCRL, prior LP/density-ratio analyses typically rely on regularization of the saddle-point problem, auxiliary function classes or occupancy-validity constraints, or a separate policy-extraction step that depends on the unknown data distribution \(\mu_D\) or on a reference policy.

\paragraph{Constrained LP-based offline RL}
In the constrained setting, COptiDICE is the closest empirical DICE-style
method, extending stationary-distribution correction to offline constrained RL
\parencite{lee2022coptidice}.
The closest LP-based partial-coverage theory result is the primal-dual method
of \textcite{zhang2024safe}, which handles offline convex CMDPs under partial
coverage.
In contrast to PDOCRL, these constrained LP/density-ratio approaches either use an additional policy-extraction or weighted behavior-cloning step, or require auxiliary structure beyond the density-ratio and critic classes considered here.

\section{Preliminaries}

\paragraph{Notations}
For functions $F : \mathcal{S} \times \mathcal{A} \rightarrow \mathbb{R}$ and $G : \mathcal{S} \rightarrow \mathbb{R}$, we denote their vector representations as $\bm{F} \in \mathbb{R}^{\vert \mathcal{S} \times \mathcal{A} \vert}$ and $\bm{G} \in \mathbb{R}^{\vert \mathcal{S} \vert}$, respectively.
For a function $Q : \mathcal{S} \times \mathcal{A} \rightarrow \mathbb{R}$ and a policy $\pi: \mathcal{S} \rightarrow \Delta(\mathcal{A})$, we write $Q(s, \pi) = \sum_a \pi(a \mid s) Q(s, a)$.
Given reward functions $r_0, r_1, \ldots, r_I$, and $\bm\lambda \in \mathbb{R}^I$, we write $r_{\bm\lambda} = r_0 + \sum_{i=1}^I \lambda_i r_i$.
We write $B \Delta^I = \{ \lambda \in \mathbb{R}_+^I : \Vert \lambda \Vert_1 \leq B\}$.
We say $(\widehat x, \widehat y) \in \mathcal{X} \times \mathcal{Y}$ is a $\xi$-near saddle point of $L(x, y)$ if $L(x, \widehat y) \leq L(\widehat x, y) + \xi$ for all $(x, y) \in \mathcal{X} \times \mathcal{Y}$.
\subsection{Constrained reinforcement learning}

We formulate the constrained RL as a discounted CMDP $\mathcal M = (\mathcal S,\mathcal A,P,r_0,\{r_i\}_{i=1}^I,\gamma,d_0)$, where $\mathcal{S}$ is the state space, $\mathcal{A}$ is the action space, $P(\cdot\mid s,a)$ is the transition kernel, $\gamma\in(0,1)$ is the discount factor, $d_0$ is the initial-state distribution, $r_0:\mathcal S\times\mathcal A\to[0,1]$ is the primary reward, and $r_1,\ldots,r_I:\mathcal S\times\mathcal A\to[0,1]$ are auxiliary reward signals defining the constraints.
For simplicity, we assume that $d_0$ is concentrated on a single start state $s_0$.

A stationary policy $\pi : \mathcal{S} \rightarrow \Delta(\mathcal{A})$ and the transition kernel $P$ induce a trajectory distribution $P^\pi$ over $(s_0,a_0,s_1,a_1,\dots)$, with expectation denoted by $\mathbb{E}^\pi$.
Action-value functions are defined as
$$
 Q_i^\pi(s,a) \coloneqq \mathbb{E}^\pi\!\left[\sum_{t=0}^{\infty} \gamma^t r_i(s_t,a_t)\mid s_0=s,a_0=a\right],\qquad i=0,1,\ldots,I.
$$
Define the discounted return for each reward signal by
\[
 J_i(\pi) \coloneqq \mathbb{E}^\pi\!\left[\sum_{t=0}^{\infty} \gamma^t r_i(s_t,a_t)\right],\qquad i=0,1,\ldots,I.
\]
The CMDP objective, given user-specified thresholds $\tau_1,\ldots,\tau_I$, is
\begin{equation}
\label{eq:cmdp}
\max_{\pi} \; J_0(\pi)
\qquad \text{subject to} \qquad
(1-\gamma)J_i(\pi) \ge \tau_i,\quad i=1,\ldots,I.
\end{equation}

\paragraph{Occupancy measures and LP formulation}
The linear programming (LP) formulation~\cite{manne1960linear} for solving~\eqref{eq:cmdp} aims to find an \textit{occupancy measure} that maximizes value, rather than optimizing directly over the policy space.
The occupancy measure of a policy $\pi$, denoted $\mu^\pi \in \Delta(\mathcal{S} \times \mathcal{A})$, is defined as:
\[
 \mu^\pi(s,a) \coloneqq (1-\gamma)\sum_{t=0}^{\infty} \gamma^t P^\pi(s_t=s,a_t=a).
\]
It satisfies $\langle \bm\mu^\pi, \bm{r}_i \rangle = (1-\gamma)J_i(\pi)$, $i=0,1,\ldots,I$.
Let $\bm{E}$ denote the linear operator defined by $[\bm{E}^\top \bm\mu](s)=\sum_a \mu(s,a)$, and let $[\bm{P}^\top \bm\mu](s') = \sum_{s,a} P(s'\mid s,a)\mu(s,a)$. Then \eqref{eq:cmdp} is equivalent to the linear program \parencite{puterman2014markov}
\begin{equation}
\label{eq:cmdp-lp}
\begin{aligned}
\max_{\bm\mu\ge 0}\quad & \langle \bm\mu, \bm{r}_0 \rangle \\
\text{subject to}\quad
& \bm{E}^\top \bm\mu = (1-\gamma)\bm{d}_0 + \gamma \bm{P}^\top \bm\mu, \\
& \langle \bm\mu, \bm{r}_i \rangle \ge \tau_i,\qquad i=1,\ldots,I.
\end{aligned}
\end{equation}
The first constraint, known as the Bellman flow constraint, ensures $\bm\mu$ corresponds to the occupancy measure of some policy.
We call such a $\bm\mu$ an \textit{admissible} occupancy measure.

\paragraph{Policy extraction}
Given an occupancy measure $\mu$, not necessarily admissible, we can extract a policy denoted $\pi(\mu)$ by normalizing $\mu$ at each state:
\begin{equation} \label{eqn:policy-extraction}
 [\pi(\mu)](a | s) \coloneqq \frac{\mu(s, a)}{\sum_{a'} \mu(s, a')}~~\text{if}~\sum_{a'} \mu(s, a') > 0,\quad \frac{1}{\vert \mathcal{A} \vert}~\text{otherwise}.
\end{equation}
Given an admissible occupancy measure $\mu$, we can recover the policy that induces $\mu$ by the policy extraction rule.
Hence, we can find an optimal policy by extracting from the optimal occupancy measure $\mu^\ast$ that solves the linear program~\eqref{eq:cmdp-lp}.

\subsection{Offline dataset and concentrability coefficient}

We observe an offline dataset $\mathcal D = \{(s_j,a_j,s_j')\}_{j=1}^n$, where $(s_j,a_j)$ are sampled i.i.d. from an unknown data distribution $\mu_D$ and $s_j'\sim P(\cdot\mid s_j,a_j)$.
We define the concentrability coefficient of a policy $\pi$ with respect to $\mu_D$ as
$$
C^\pi \coloneqq \Vert \mu^\pi / \mu_D \Vert_\infty = \max_{s, a} \mu^\pi(s, a) / \mu_D(s, a),
$$
which measures how well the dataset covers the state-action pairs visited by the policy $\pi$.
Intuitively, if $C^\pi$ is large, then the dataset does not cover the state-action pairs visited by $\pi$, making it difficult to learn about $\pi$ from the dataset.

\subsection{Unconstrained Reinforcement Learning}

\paragraph{Markov Decision Processes}
We formulate unconstrained RL with a Markov decision process (MDP)~\parencite{puterman2014markov} defined by a tuple $\mathcal{M} = \left( \mathcal{S}, \mathcal{A}, P, r, \gamma, d_0 \right)$, where $\mathcal{S}$ is the state space, $\mathcal{A}$ is the action space, $P : \mathcal{S} \times \mathcal{A} \rightarrow \Delta(\mathcal{S})$ is the transition probability kernel, $r : \mathcal{S} \times \mathcal{A} \rightarrow [0, 1]$ is the reward function, $\gamma \in (0, 1)$ is the discount factor, and $d_0 \in \Delta(\mathcal{S})$ is the initial state distribution.
For simplicity, we assume $d_0$ is concentrated on a single state $s_0 \in \mathcal{S}$.
We assume the reward function $r$ is known to the learner while $P$ is unknown.
The interaction protocol between the environment and the agent under a MDP is as follows.
The environment samples an initial state from $d_0$.
Then, for each time step $t = 0, 1, \dots$, the agent chooses an action $a_t$ and receives the reward $r(s_t, a_t)$.
Then the environment advances to the next state $s_{t + 1}$ sampled from $P(\cdot | s_t, a_t)$.

\paragraph{Policies and Value Functions}
A (stationary) policy $\pi : \mathcal{S} \rightarrow \Delta(\mathcal{A})$ maps each state to a probability distribution over the action space.
A policy $\pi$, together with the transition probability kernel $P$, induces a distribution denoted by $P^\pi$ on the trajectory $(s_0, a_0, s_1, a_1, \dots)$.
We denote by $\mathbb{E}^\pi$ the expectation over this distribution.
The value of a policy $\pi$ is defined as
$$
 J(\pi) \coloneqq \mathbb{E}^\pi \left[ \sum_{t = 0}^\infty \gamma^t r(s_t, a_t) \right].
$$
The state value function and state-action value function of a policy $\pi$ are defined as
$$
 V^\pi(s) \coloneqq \mathbb{E}^\pi \left[ \sum_{t = 0}^\infty \gamma^t r(s_t, a_t) | s_0 = s \right],\quad
Q^\pi(s, a) \coloneqq \mathbb{E}^\pi \left[ \sum_{t = 0}^\infty \gamma^t r(s_t, a_t) | s_0 = s, a_0 = a \right].
$$
The goal of unconstrained RL is to find a stationary policy $\pi$ that maximizes the value $J(\pi)$:
\begin{equation} \label{eqn:rl}
\max_\pi J(\pi)
\end{equation}

\subsection{Linear Programming Formulation of Unconstrained Reinforcement Learning}
The linear programming (LP) formulation~\cite{manne1960linear} for solving~\eqref{eqn:rl} aims to find an \textit{occupancy measure} that maximizes value, rather than optimizing directly over the policy space. The occupancy measure of a policy $\pi$, denoted $\mu^\pi \in \Delta(\mathcal{S} \times \mathcal{A})$, is defined as:
$$
 \mu^\pi(s, a) \coloneqq (1 - \gamma) \sum_{t = 0}^\infty \gamma^t P^\pi(s_t = s, a_t = a).
$$

Leveraging the relation $(1 - \gamma) J(\pi) = \langle \bm\mu^\pi, \bm{r} \rangle$, the LP optimizes $\langle \bm\mu, \bm{r} \rangle$ over admissible occupancy measures $\bm\mu$:
\begin{equation} \label{eqn:lp}
\begin{aligned}
\max_{\bm\mu \geq 0}&~~~ \langle \bm\mu, \bm{r} \rangle \\
\text{subject to}&~~~ \bm{E}^\top \bm\mu = (1 - \gamma) \bm{d}_0 + \gamma \bm{P}^\top \bm\mu.
\end{aligned}
\end{equation}
where the optimization variable is $\bm\mu \in \mathbb{R}^{\vert \mathcal{S} \times \mathcal{A} \vert}$.
The bold-faced variables $\bm{r} \in \mathbb{R}^{\vert \mathcal{S} \times \mathcal{A} \vert}$, $\bm{d}_0 \in \mathbb{R}^{\vert \mathcal{S} \vert}$ and $\bm{P} \in \mathbb{R}^{\vert \mathcal{S} \times \mathcal{A} \vert \times \vert \mathcal{S} \vert}$ denote the vector or matrix representations of the functions $r$, $d_0$ and $P$, respectively.
The linear operator $\bm{E}$ is such that $[\bm{E}^\top \mu](s) = \sum_{a'} \mu(s, a')$ and $[\bm{E}V](s, a) = V(s)$ for all $s, a$.
The constraint, known as the Bellman flow constraint, ensures $\bm\mu$ is an admissible occupancy measure induced by some policy.

\paragraph{Policy Extraction}
Consider a procedure for extracting policy from an occupancy measure $\mu$, not necessarily admissible:
\begin{equation} \label{eqn:policy-extraction}
[\pi(\mu)](a | s) \coloneqq \frac{\mu(s, a)}{\sum_{a'} \mu(s, a')}~~\text{if}~\sum_{a'} \mu(s, a') > 0,\quad \frac{1}{\vert \mathcal{A} \vert}~\text{otherwise}.
\end{equation}
It is known that given an admissible occupancy measure $\mu$, the extracted policy $\pi(\mu)$ induces the occupancy measure $\mu$.
With this fact, we can find an optimal occupancy measure $\mu^\ast$ by solving the linear program to find optimal $\mu^\ast$, and extracting policy $\pi^\ast = \pi(\mu^\ast)$ from it.

\subsection{Offline Dataset}

In offline constrained RL, we assume access to an offline dataset $\mathcal{D} = \{(s_j, a_j, s_j') \}_{j = 1}^n$ where $(s_j, a_j)$ are generated i.i.d. from a data distribution $\mu_D$ and $s_j'$ is sampled from $P(\cdot~|~s_j, a_j)$.
Such an i.i.d. assumption on the offline dataset is commonly made in the offline RL literature \parencite{xie2021bellman,zhan2022offline,chen2022offline,zhu2023importance} to facilitate analysis of concentration bounds.
Additionally, we make the following concentrability assumption that bounds the ratio between the occupancy measure of the optimal policy and the data distribution, which is necessary for learning an optimal policy from the dataset.

\begin{assumption}[Single-Policy Concentrability] \label{assumption:concentrability}
For an optimal policy $\pi^\ast$, we have $\Vert \mu^{\pi^\ast} / \mu_D \Vert_\infty \leq C^\ast$ and such an upper bound $C^\ast$ is known.
\end{assumption}
This assumption says that the data distribution $\mu_D$ covers the state-action pairs visited by an optimal policy $\pi^\ast$, and the concentrability coefficient that measures the coverage ratio is bounded by $C^\ast$.
If the concentrability coefficient is large, then there are state-action pairs visited by $\pi^\ast$ that rarely appear in the dataset, making sample efficient learning hard.

\section{Main results for offline constrained RL} \label{sec:pdocrl}

We start from the LP formulation of the offline CMDP in \eqref{eq:cmdp-lp}.
Introducing a Lagrange multipliers \(\bm{V} \in \mathbb{R}^{\vert \mathcal{S} \vert}\) for the Bellman flow constraints and \(\bm\lambda \in \mathbb{R}_+^I\) for the safety constraints gives the Lagrangian
\begin{equation} \label{eqn:lagrangian}
 L(\bm\mu; \bm{V}, \bm\lambda) \coloneqq \langle \bm\mu, \bm{r} \rangle + \langle \bm{V}, (1 - \gamma) \bm{d}_0 + \gamma \bm{P}^\top \bm\mu - \bm{E}^\top \bm\mu \rangle + \sum_{i=1}^I \lambda_i (\langle \bm\mu, \bm{r}_i \rangle - \tau_i),
\end{equation}
which is linear in the occupancy measure \(\mu\).
Writing \(\mu(s,a)=w(s,a)\mu_D(s,a)\) turns the \(\mu\)-dependent term into an expectation under the data distribution \(\mu_D\), allowing it to be estimated directly from the offline dataset by sample averaging:
\begin{equation} \label{eqn:ablation}
 \widehat{L}(\bm w; \bm V, \bm \lambda) \coloneqq (1-\gamma)V(s_0)+\frac{1}{n} \sum_{j=1}^n w(s_j, a_j) \left( r_{\bm\lambda}(s_j, a_j) + \gamma V(s_j') - V(s_j) \right) - \sum_{i=1}^I \lambda_i \tau_i,
\end{equation}
where $\bm{r}_{\bm\lambda} = \bm{r}_0 + \sum_{i=1}^I \lambda_i \bm{r}_i$ is the Lagrangian reward.
If \(\mu_D\) were known, a natural approach would be to optimize over the density ratio \(w\), form the induced occupancy measure $\mu = w \mu_D$, and then recover a policy by the statewise normalization in using \eqref{eqn:policy-extraction}.
Appendix~\ref{appendix:known-data-distribution} shows that this approach already admits provable guarantees.
In practice, however, \(\mu_D\) may be unknown, and policy extraction from $w$ becomes nontrivial.

Existing empirical methods~\parencite{lee2021optidice,lee2022coptidice} address this issue by adding an importance-weighted behavior-cloning step motivated by the identity $\pi^\ast = \argmax_\pi \mathbb{E}_{(s, a) \sim \mu^D}[w^{\pi^\ast}(s, a)\log \pi(a | s)]$:
\begin{equation} \label{eqn:policy-extraction-density-ratio}
 \widehat\pi = \argmax_\pi \frac{1}{n} \sum_{j = 1}^n \widehat{w}(s_j, a_j) \log \pi(a_j | s_j),
\end{equation}
where $\widehat{w}$ is the learned density ratio.
However, this approach suffers from a fundamental statistical limitation: empirical policy extraction from density ratios can be unreliable under finite samples, even when the underlying density ratio is well-behaved. The following proposition formalizes this failure mode.
See Appendix~\ref{sec:policy-extraction} for a proof.
\begin{proposition} \label{prop:policy-extraction}
There exist a MDP $(\mathcal{S}, \mathcal{A}, P, r, \gamma, d_0)$ and a policy class $\Pi$ that contains an optimal policy with $\vert \mathcal{S} \vert = \vert \mathcal{A} \vert = \vert \Pi \vert = 2$, data distribution $\mu_D$, and a density-ratio estimate $\widehat{w}$ such that:
\begin{enumerate}[label=(\roman*),nosep]
\item Policy $\widehat{\pi} = \pi(\widehat{w} \, \mu_D)$ extracted from $\widehat{w}$ with the knowledge of $\mu_D$ is optimal.
\item The policy extraction $\widehat{\pi} = \argmax_{\pi \in \Pi} \mathbb{E}_{\mu_D}[\widehat{w}(s, a) \log \pi(a | s)]$ is optimal.
\item But the empirical policy extraction \eqref{eqn:policy-extraction-density-ratio} returns a suboptimal policy with constant probability.
\end{enumerate}
\end{proposition}

This negative result suggests that policy extraction from density ratios is inherently fragile.
It motivates designing algorithms that optimize over policies directly, without requiring a separate policy-extraction step.
In the next section, we introduce a Lagrangian decomposition and reparameterization that make the policy an explicit optimization variable, thereby avoiding the policy-extraction step.

\subsection{Lagrangian decomposition and reparameterization}

The key difficulty is that the occupancy measure variable $\mu$ in the Lagrangian formulation~\eqref{eqn:lagrangian} plays two roles: Lagrangian estimation and policy extraction.
For Lagrangian estimation, we use the change-of-measure $\mu = w \mu_D$ which allows the Lagrangian to be expressed as an expectation under the data distribution $\mu_D$.
After the change of variables, however, policy extraction has to be done on $w \mu_D(s,a)$, and hence requires knowledge of the unknown data distribution $\mu_D$.

Lagrangian decomposition addresses this issue by introducing a copy variable $\nu$ of $\mu$ and separating these two roles: \(\mu\) is used for Lagrangian estimation through the change of measure, while \(\nu\) is used for policy extraction.
The decomposition is obtained by adding a constraint $\nu = \mu$ to the original LP formulation \eqref{eq:cmdp-lp} and decomposing the Bellman flow constraint:
\begin{equation}
\begin{aligned}
\max_{\bm\mu\ge 0,\ \bm\nu\ge 0}\quad & \langle \bm\mu, \bm{r}_0 \rangle \\
\text{subject to}\quad
& \bm{E}^\top \bm\nu = (1-\gamma)\bm{d}_0 + \gamma \bm{P}^\top \bm\mu, \\
& \bm\mu = \bm\nu, \\
& \langle \bm\mu, \bm{r}_i \rangle \ge \tau_i,\qquad i=1,\ldots,I.
\end{aligned}
\end{equation}
Introducing the multiplier $\bm{Q} \in \mathbb{R}^{\vert \mathcal{S} \times \mathcal{A} \vert}$ for the constraint $\bm\mu = \bm\nu$, we get the Lagrangian function
\begin{equation} \label{eqn:lagrangian-decomposition}
 L(\bm\mu, \bm\nu ; \bm{V}, \bm{Q}, \bm{\lambda}) = (1 - \gamma) \langle \bm{d}_0, \bm{V} \rangle + \langle \bm\mu, \bm{r}_{\bm\lambda} + \gamma \bm{P}\bm{V} - \bm{Q} \rangle + \langle \bm\nu, \bm{Q} - \bm{E}\bm{V} \rangle - \sum_{i=1}^I \lambda_i \tau_i.
\end{equation}
Lagrangian decomposition~\parencite{shepardson1980lagrangean,guignard1987lagrangean} is a classical method for decomposing an optimization problem into subproblems by introducing copies of optimization variables.
It is first used in the context of MDP by \textcite{mehta2009q}.
It has since been used in RL in several settings~\parencite{neu2023efficient,gabbianelli2024offline,hong2024primal,neu2024offline}.

One could now change variables by writing \(w=\mu/\mu_D\), estimate the decomposed Lagrangian as \(\widehat L(w,\nu;V,Q,\lambda)\), find a saddle point, and then extract a policy from \(\nu\).
But a simple reparameterization trick allows us to simplify the optimization problem by introducing a policy variable.
Specifically, consider a policy $\pi$ induced by $\nu$. Then, $\nu(s, a) = \pi(a | s) \nu(s)$ and the decomposed Bellman flow constraint suggests the following reparameterization of $\nu$ in terms of $\mu$ and $\pi$:
$$
\nu_{\mu,\pi}(s,a) \coloneqq \pi(a \mid s)\bigl((1-\gamma)d_0(s)+\gamma [\bm P^\top \bm\mu](s)\bigr).
$$
By construction, $\nu = \nu_{\mu, \pi}$ satisfies the decomposed Bellman flow constraint for every $\mu$ and $\pi$, and the policy extracted from $\nu_{\mu, \pi}$ is exactly $\pi$.
Substituting $\nu = \nu_{\mu, \pi}$ into the Lagrangian and using the notation $Q_\pi(s) = \sum_{a'} \pi(a' | s) Q(s, a')$ yield
\begin{equation} \label{eqn:lagrangian-decomposition-reparameterization}
 L(\bm\mu, \pi ; \bm{Q}, \bm\lambda)= (1 - \gamma) \langle \bm{Q}_\pi, \bm{d}_0 \rangle + \langle \bm\mu, \bm{r}_{\bm\lambda} + \gamma \bm{P}\bm{Q}_\pi - \bm{Q} \rangle - \sum_{i=1}^I \lambda_i \tau_i,
\end{equation}

The reparameterized decomposed Lagrangian was first introduced by \textcite{gabbianelli2024offline} for offline RL under linear MDPs.
In that setting, the linear structure further allows $\mu$ and $Q$ to be represented by linear functions of features, reducing the optimization to one over feature weights.
In the general function approximation setting, however, the restricted saddle-point problem over general function classes for $\mu$ and $Q$ can admit spurious saddle points that do not correspond to optimal policies, as discussed next, requiring additional assumptions to rule out this pathology.

\subsection{Spurious saddle points and assumptions for saddle-point optimality}

The next proposition shows that merely realizing an optimal policy, its occupancy measure, and its associated action-value functions is not sufficient to rule this out.

\begin{proposition}
\label{prop:negative2}
There exists a CMDP $\mathcal M=(\mathcal S,\mathcal A,P,r_0,\{r_i\}_{i = 1}^I,\gamma,d_0)$,
function classes \(\mathcal U \subseteq \mathbb R_+^{\mathcal S\times\mathcal A}\),
\(\mathcal Q \subseteq \mathbb R^{\mathcal S\times\mathcal A}\), and a policy class \(\Pi\) such that:
\begin{enumerate}[label=(\roman*),nosep]
\item there exists an optimal feasible policy \(\pi^\ast \in \Pi\) with occupancy measure
\(\mu^\ast \in \mathcal{U}\);
\item for all \(\alpha \in \mathbb{R}_+^I\), $Q^\ast_\alpha \in \mathcal{Q}$, where \(Q_\alpha^\ast\) is the optimal value function for
\(r_0+\sum_{i = 1}^I \alpha_i r_i\);
\item nevertheless, there exists \(\widehat{\pi}\in\Pi\) that is not optimal for the CMDP such that
\((\widehat \mu,\widehat{\pi};\widehat{Q},\widehat{\bm\lambda})\) is a saddle point of the reparameterized Lagrangian
\(L\) over \((\mathcal U \times \Pi)\times (\mathcal Q \times \mathbb R_+^I)\).
\end{enumerate}
\end{proposition}

To deal with the problem of spurious saddle points in the decomposed formulation, we first impose the following Slater-type condition to ensure the existence of a feasible policy that satisfies the constraints with a positive margin:

\begin{assumption}[Slater's condition] \label{assumption:slater}
There exist a constant $\varphi > 0$ and a policy $\pi$ that satisfy $(1 - \gamma) J_i(\pi) \geq \tau_i + \varphi$ for all $i = 1, \dots, I$.
\end{assumption}

Under a mild assumption that a feasible policy exists, the thresholds $\tau_i$ can be adjusted downward by $\varphi$ to ensure Slater's condition holds with margin $\varphi$, if necessary.
This Slater-type condition is standard in constrained RL analyses, where it is commonly used to obtain strong duality and to bound the optimal dual variable \parencite{hong2024primal,mondal2024sample,ying2025policy}.
The following lemma demonstrates that, under this assumption, an optimal dual variable $\bm\lambda^\ast$ is bounded.

\begin{lemma} \label{lemma:dual-variable-bound}
Consider a constrained optimization problem \eqref{eq:cmdp} with thresholds $\bm\tau = (\tau_1, \dots, \tau_I)$.
Suppose the problem satisfies Assumption~\ref{assumption:slater} with margin $\varphi > 0$.
Then, the optimal dual variable $\bm\lambda^\ast$ of the problem satisfies $\Vert \bm\lambda^\ast \Vert_1 \leq \frac{1}{\varphi}$.
\end{lemma}

Now, to deal with spurious saddle points, we impose the following all-policy value function realizability assumption.

\begin{assumption}[All-Policy Value Function Realizability] \label{assumption:all-policy-state-action-value-realizability}
For every policy $\pi \in \Pi$, we have $Q_0^\pi \in \mathcal{Q}$ and $Q_0^\pi + (1 + \frac{1}{\varphi}) Q_i^\pi \in \mathcal{Q}$ for all $i = 0, 1, \dots, I$.
\end{assumption}

This is a strong all-policy realizability condition, but it is standard in iterative offline actor-critic analyses with general function approximation, where the critic must evaluate the policies generated by the actor \parencite{cheng2022adversarially,zhu2023importance,wei2024adversarially}.
Single-policy realizability is possible in methods that solve saddle/minimax problems such as PRO-RL~\parencite{zhan2022offline} and CORAL~\parencite{rashidinejad2022optimal}, but these methods are not oracle-efficient in our sense.
With the stronger all-policy condition the lemma below shows spurious saddle points vanish.
We leave the question of whether weaker assumptions for iterative methods can rule out spurious saddle points to future work.

\begin{lemma} \label{lemma:saddle-point-lagrangian-decomposition}
Assume Assumption~\ref{assumption:slater} and ~\ref{assumption:all-policy-state-action-value-realizability}.
Consider function classes \(\mathcal{U} \subseteq \mathbb{R}_+^{\vert \mathcal{S} \times \mathcal{A} \vert}\) and
\(\mathcal{Q} \subseteq \mathbb{R}^{\vert \mathcal{S} \times \mathcal{A} \vert}\).
Let \(\Pi\) be a policy class containing an optimal feasible policy \(\pi^\ast\), and let
\(\mu^\ast = \mu^{\pi^\ast} \in \mathcal{U}\).
If \((\widehat\mu, \widehat\pi; \widehat Q, \widehat\lambda)\) is a saddle point of \(L\) over
$(\mathcal{U} \times \Pi) \times (\mathcal{Q} \times (1 + \frac{1}{\varphi}) \Delta^I)$, then \(\widehat\pi\) is an optimal feasible policy.
\end{lemma}

\subsection{Lagrangian estimation and minimax algorithm}

As discussed previously, we can change variable $w = \mu / \mu_D$
to turn the $\mu$-dependent term in the Lagrangian~\eqref{eqn:lagrangian-decomposition-reparameterization} into an expectation under $\mu_D$, which can be estimated by sample averaging from the offline dataset as:
\[
\widehat{L}(\bm{w}, \pi ; \bm{Q}, \bm\lambda)
\coloneqq
(1 - \gamma) Q(s_0, \pi)
+
\frac{1}{n} \sum_{j = 1}^n w(s_j, a_j)
\Bigl(
r_{\bm\lambda}(s_j, a_j)
+ \gamma Q(s_j', \pi)
- Q(s_j, a_j)
\Bigr)
-
\sum_{i=1}^I \lambda_i \tau_i.
\]
Motivated by this estimator, we first consider the following idealized empirical saddle-point problem:
\begin{equation} \label{eqn:primal-dual-lagrangian-decomposition}
\max_{w \in \mathcal{W}, \pi \in \Pi}~\min_{Q \in \mathcal{Q}, \lambda \in (1 + \frac{1}{\varphi}) \Delta^I} \widehat{L}(w, \pi ; Q, \lambda),
\end{equation}
and output the policy component $\widehat{\pi}$ of a saddle point $(\widehat{w}, \widehat\pi ; \widehat{Q}, \widehat{\lambda})$.
Although solving \eqref{eqn:primal-dual-lagrangian-decomposition} exactly is generally computationally intractable, it provides a clean benchmark for the analysis: under suitable realizability and boundedness assumptions, an exact saddle point of $\widehat{L}$ yields a near-optimal and near-feasible policy.
We first require that the density ratio of an optimal feasible policy belong to $\mathcal{W}$.

\begin{assumption} \label{assumption:miw-realizability}
For an optimal policy $\pi^\ast$, we have $w^{\pi^\ast} = \mu^{\pi^\ast} / \mu_D \in \mathcal{W}$.
\end{assumption}

To control the statistical error of $\widehat{L}$, we also impose standard boundedness assumptions on the density-ratio and $Q$-function classes.

\begin{assumption} \label{assumption:boundedness}
For all $w \in \mathcal{W}$, $\Vert w \Vert_\infty \leq C^\ast$ and for all $Q \in \mathcal{Q}$, $\Vert Q \Vert_\infty \leq (2 + \frac{1}{\varphi}) / (1 - \gamma)$.
\end{assumption}

The bound $(2 + \frac{1}{\varphi}) / (1 - \gamma)$ on the $Q$-functions is natural since the Lagrangian reward $r_{\bm\lambda}$ is bounded by $2 + \frac{1}{\varphi}$.
The bound $C^\ast$ on $\mathcal{W}$, together with the assumption $w^{\pi^\ast} \in \mathcal{W}$, implies $C^{\pi^\ast} = \Vert w^{\pi^\ast} \Vert_\infty \leq C^\ast$.
Thus, positing a density-ratio class $\mathcal{W}$ bounded by $C^\ast$ amounts to assuming that the learner knows an upper bound on the concentrability coefficient of the optimal policy. Such an assumption is common in offline RL with function approximation.
Appendix~\ref{sec:saddle-point-analysis} contains the guarantee for $\widehat{\pi}$ obtained by the minimax algorithm \eqref{eqn:primal-dual-lagrangian-decomposition}.

\subsection{Oracle-efficient primal-dual algorithm}

\begin{algorithm}[t] \label{alg:pdocrl}
\KwInput{Dataset $\mathcal{D} = \{ (s_j, a_j, s_j') \}_{j = 1}^n$, function classes $\mathcal{W}$, $\mathcal{Q}$ and $\Pi$, no-regret online linear-optimization oracle $\texttt{OLO}_{\mathcal{W}}$, no-regret policy optimization oracle $\texttt{PO}_{\Pi, \mathcal{Q}}$, linear minimization oracles $\texttt{LMO}_{\mathcal{Q}}$ and $\texttt{LMO}_{B\Delta^I}$, dual bound $B > 0$, number of iterations $T$.}
\For{$t = 1, \dots, T$}{
$w_t \gets \texttt{OLO}_{\mathcal{W}}(\{ \nabla_w \widehat{L}(~\cdot~, \pi_\tau ; Q_\tau, \lambda_\tau) \}_{\tau = 1}^{t-1})$. \\
$\pi_t \gets \texttt{PO}_{\Pi, \mathcal{Q}}(Q_1, \dots, Q_{t - 1})$. \\
$\lambda_t \gets \texttt{LMO}_{B\Delta^I} (\widehat{L}(w_t, \pi_t ; Q_{t - 1}, ~\cdot~))$. \label{algline:lambda-update} \\
$Q_t \gets \texttt{LMO}_{\mathcal{Q}} (\widehat{L}(w_t, \pi_t ; ~\cdot~, \lambda_t))$. \label{algline:q-update}
}
\KwReturn{$\text{Uniform}(\pi_1, \dots, \pi_T)$}
\caption{PDOCRL}
\end{algorithm}  

We have not yet addressed the computational challenges of finding the saddle point.
In this section, we propose an oracle-efficient primal-dual algorithm that solves the saddle point problem~\eqref{eqn:primal-dual-lagrangian-decomposition} by iteratively optimizing for the primal variables $w, \pi$ and the dual variables $Q, \lambda$, using standard oracles.

The \(w\)-player employs a no-regret online linear optimization algorithm over $\mathcal{W}$ defined as follows.
\begin{definition}
An online algorithm $\texttt{OLO}_\mathcal{W}$ is called a \emph{no-regret online linear optimization oracle} for the function class $\mathcal{W} \subseteq \mathbb{R}_+^{\mathcal{S} \times \mathcal{A}}$,
if for every sequence $c_1, \dots, c_T \in \mathcal{W}$, the algorithm sequentially outputs $w_t = \texttt{OLO}_\mathcal{W}(c_1, \dots, c_{t - 1}) \in \mathcal{W}$, $t=1,\dots,T$,
such that, for every comparator $w \in \mathcal{W}$,
$$
 \frac{1}{T} \sum_{t=1}^T \langle c_t, w - w_t \rangle \rightarrow 0, \quad \text{as}~~T \rightarrow \infty.
$$
\end{definition}
For example, with an assumption that $\mathcal{W}$ is convex, the projected gradient ascent can serve as the no-regret oracle, as discussed in Appendix~\ref{sec:w-player}.

The $\pi$-player employs a no-regret policy optimization oracle over $\Pi$ defined as follows.

\begin{definition} \label{def:no-regret-oracle}
An online algorithm $\texttt{PO}_{\Pi, \mathcal{Q}}$ is called a \emph{no-regret policy optimization oracle} for the \(\pi\)-player, function class \(\mathcal Q\subseteq \mathbb R^{\mathcal S\times\mathcal A}\) and policy class $\Pi \subseteq (\Delta(\mathcal{A}))^{\mathcal{S}}$,
if for every sequence \(Q_1,\dots,Q_T \in \mathcal Q\), the algorithm sequentially outputs policies
$\pi_t = \texttt{PO}_{\Pi, \mathcal{Q}}(Q_1,\dots,Q_{t-1}) \in \Pi$, $t=1,\dots,T$,
such that, for every comparator policy $\pi \in \Pi$,
\[
 \frac{1}{T} \sum_{t=1}^T \mathbb{E}_\pi [ Q_t(s, \pi) - Q_t(s, \pi_t)] \rightarrow 0, \quad \text{as}~~T \rightarrow \infty,
\]
where $\mathbb{E}_\pi$ denotes the expectation over $s \sim \sum_a \mu^\pi(\cdot, a)$, the state occupancy measure of $\pi$.
\end{definition}

As discussed by \textcite{wei2024adversarially}, there exist many methods that can serve as the no-regret oracle, for example, the mirror descent approach~\parencite{geist2019theory} and natural policy gradient approach~\parencite{kakade2001natural} of the form $\pi_t(\cdot | s) \propto \pi_{t - 1}(\cdot | s) \exp(\alpha Q_{t - 1}(s, \cdot))$ for some learning rate $\alpha > 0$.

For the $Q$-player, since the Q-dependent part of $\widehat{L}$ is linear in Q, we use a linear minimization oracle $\texttt{LMO}_{\mathcal{Q}}(c) \in \argmin_{Q \in \mathcal{Q}} \langle c, Q \rangle$.
Similarly, since the $\lambda$-dependent part of $\widehat{L}$ is linear in $\lambda$, we use a linear minimization oracle $\texttt{LMO}_{B\Delta^I}(c) \in \argmin_{\lambda \in B\Delta^I} \langle c, \lambda \rangle$.

The full algorithm is described in Algorithm~\ref{alg:pdocrl}. The algorithm runs for $T$ iterations and returns a policy sampled uniformly at random from $\pi_1, \dots, \pi_T$.
The policy returned by the algorithm has the following guarantee.
For the exact statement and the proof, see Theorem~\ref{theorem:crl-pd} in Appendix~\ref{appendix:unknown-data-distribution}.

\begin{theorem}[Informal] \label{theorem-crl-pd}
Under Assumptions~\ref{assumption:slater}, \ref{assumption:all-policy-state-action-value-realizability}, \ref{assumption:miw-realizability} and \ref{assumption:boundedness}, running Algorithm~\ref{alg:pdocrl} with dual bound $B = 1 + \frac{1}{\varphi}$ and $T$ large enough, the output policy $\widehat\pi$ satisfies
$$
(1 - \gamma) J_0(\widehat\pi) \geq (1 - \gamma) J_0(\pi^\ast) - \varepsilon_n, \quad (1 - \gamma) J_i(\widehat\pi) \geq \tau_i - \varepsilon_n,~~i = 1, \dots, I.
$$
with probability at least $1 - \delta$,
where $\varepsilon_n = \widetilde{\mathcal{O}}(C^\ast \sqrt{\log (I \vert \mathcal{W} \vert \vert \Pi \vert \vert \mathcal{Q} \vert / \delta) / n})$.
\end{theorem}


\subsection{Adaptation to offline unconstrained RL setting}

The algorithms we have discussed so far for the constrained RL setting can be adapted to the unconstrained offline RL setting by removing the dual variable $\lambda$ and the associated terms in the Lagrangian, and the analysis can be adapted accordingly.

In fact, with a careful analysis on the guarantee of the saddle point of the Lagrangian, we can get a stronger guarantee in the unconstrained setting.
Specifically, we can compete with any comparator policy covered by $\mathcal{W}$, instead of just the optimal policy as in the constrained setting.
The adapted algorithm is shown in Algorithm~\ref{alg:pdorl}, and the guarantee is stated in Theorem~\ref{theorem:pd-unconstrained}.

\begin{algorithm} \label{alg:pdorl}
\KwInput{Dataset $\mathcal{D} = \{ (s_j, a_j, s_j') \}_{j = 1}^n$, function classes $\mathcal{W}$, $\mathcal{Q}$ and $\Pi$, no-regret online linear-optimization oracle $\texttt{OLO}_{\mathcal{W}}$, no-regret policy optimization oracle $\texttt{PO}_{\Pi, \mathcal{Q}}$, linear minimization oracles $\texttt{LMO}_{\mathcal{Q}}$, number of iterations $T$.}
\For{$t = 1, \dots, T$}{
$w_t \gets \texttt{OLO}_{\mathcal{W}}(\{ \nabla_w \widehat{L}(~\cdot~, \pi_\tau ; Q_\tau) \}_{\tau = 1}^{t-1})$. \\
$\pi_t \gets \texttt{PO}_{\Pi, \mathcal{Q}}(Q_1, \dots, Q_{t - 1})$. \\
$Q_t \gets \texttt{LMO}_{\mathcal{Q}}(\nabla_Q \widehat{L}(w_t, \pi_t ; ~\cdot~))$. \label{algline:argmin}
}
\KwReturn{$\text{Uniform}(\pi_1, \dots, \pi_T)$}
\caption{PDORL}
\end{algorithm}

\begin{theorem}[Informal] \label{theorem:pd-unconstrained}
Assume $Q^\pi \in \mathcal{Q}$ for all $\pi \in \Pi$, and $w^{\pi_c} \in \mathcal{W}$ for some comparator policy $\pi_c \in \Pi$, and $\Vert w \Vert_\infty \leq C$ for all $w \in \mathcal{W}$ and $\Vert Q \Vert_\infty \leq 1 / (1 - \gamma)$ for all $Q \in \mathcal{Q}$.
Running Algorithm~\ref{alg:pdorl} with $T$ large enough, the output policy $\widehat\pi$ satisfies
$$
(1 - \gamma) J(\widehat\pi) \geq (1 - \gamma) J(\pi_c) - \varepsilon_n
$$
with probability at least $1 - \delta$, where $\varepsilon_n = \mathcal{O}(\frac{C}{1 - \gamma} \sqrt{\log (\vert \mathcal{W} \vert \vert \Pi \vert \vert \mathcal{Q} \vert / \delta) / n})$.
\end{theorem}
\section{Experiments} \label{sec:exp}

We compare PDOCRL with  offline constrained RL algorithms~\cite{liu2023datasetsbenchmarksofflinesafe}, assessing constraint violation and reward maximization.\footnotemark[1]\footnotetext[1]{Source code: \url{https://github.com/komin0407/PDOCRL}} 
For each task, training is performed for 100,000 gradient steps using three cost thresholds and three random seeds.
Each trained model is then evaluated over 20 episodes, and we report the average return.
See Appendix \ref{appendix:exp} for details on the experimental setup.

\paragraph{Benchmark and datasets}
The experiments are conducted on the BulletGym benchmark~\cite{gronauer2022bullet}, consisting of continuous-control tasks with explicit safety constraints.
We use the offline datasets provided by the DSRL package~\cite{liu2023datasetsbenchmarksofflinesafe}, which offers D4RL-styled~\cite{fu2020d4rl} datasets for safe RL tasks.
For each BulletGym task, the dataset is collected by training multiple safe RL expert policies under varying cost thresholds, then merging the resulting trajectories. 
We follow the standard evaluation protocol of~\cite{liu2023datasetsbenchmarksofflinesafe}, reporting normalized reward and normalized cost, where the normalized cost is defined as the ratio of the evaluated cost return to the cost threshold. 

\paragraph{Baselines}
The baselines are taken from the benchmark paper \textcite{liu2023datasetsbenchmarksofflinesafe}.
BC-All and BC-Safe~\cite{liu2022constrained, xu2022trustworthy} are imitation-learning-based methods.
CDT~\cite{liu2022constrained} uses a transformer architecture conditioned on return and cost.
BCQ-Lag~\cite{xu2022constraints}, BEAR-Lag~\cite{xu2022constraints}, and CPQ~\cite{xu2022constraints} are Q-learning-based methods, while COptiDICE~\cite{lee2022coptidice} is a LP-based approach, most closely related to our method.

\paragraph{Evaluation metrics} We evaluate each method using the normalized reward return and normalized cost return. The reward is normalized by the empirical minimum and maximum reward values of each task, making the results comparable across different environments. The cost is normalized by the target cost threshold, so that a value below 1 indicates that the safety constraint is satisfied. For simplicity, we refer to the normalized reward return and normalized cost return as \emph{reward} and \emph{cost}, respectively. For more details, refer to \ref{appendix:metric}.

\begin{table}[htb]
\caption{\small Evaluation results of the normalized reward and cost on BulletGym tasks. The cost threshold is 1.
$\uparrow$: higher reward is better. $\downarrow$: lower cost (up to threshold 1) is better.
\textbf{Bold}: Safe agents (normalized cost $\le 1$).
{\color[HTML]{656565} Gray}: Unsafe agents.
{\color[HTML]{0000FF} \textbf{Blue}}: Safe agent with the highest reward.}
\vspace{-2pt}
\label{tab:bulletgym-result}
\begin{center}
\resizebox{1.\linewidth}{!}{
\begin{tabular}{c cc cc cc cc cc cc cc cc}
\toprule
\multirow{2}{*}{Task}
& \multicolumn{2}{c}{BC-All}
& \multicolumn{2}{c}{BC-Safe}
& \multicolumn{2}{c}{CDT}
& \multicolumn{2}{c}{BCQ-Lag}
& \multicolumn{2}{c}{BEAR-Lag}
& \multicolumn{2}{c}{CPQ}
& \multicolumn{2}{c}{COptiDICE}
& \multicolumn{2}{c}{\textbf{PDOCRL (Ours)}} \\
\cmidrule(lr){2-3} \cmidrule(lr){4-5} \cmidrule(lr){6-7} \cmidrule(lr){8-9} \cmidrule(lr){10-11} \cmidrule(lr){12-13} \cmidrule(lr){14-15} \cmidrule(lr){16-17}
& Rew.\,$\uparrow$ & Cost\,$\downarrow$
& Rew.\,$\uparrow$ & Cost\,$\downarrow$
& Rew.\,$\uparrow$ & Cost\,$\downarrow$
& Rew.\,$\uparrow$ & Cost\,$\downarrow$
& Rew.\,$\uparrow$ & Cost\,$\downarrow$
& Rew.\,$\uparrow$ & Cost\,$\downarrow$
& Rew.\,$\uparrow$ & Cost\,$\downarrow$
& Rew.\,$\uparrow$ & Cost\,$\downarrow$ \\
\midrule
 
BallRun
& {\color[HTML]{656565} $0.60_{\scriptscriptstyle\pm0.10}$} & {\color[HTML]{656565} $5.08_{\scriptscriptstyle\pm0.74}$}
& {\color[HTML]{656565} $0.27_{\scriptscriptstyle\pm0.14}$} & {\color[HTML]{656565} $1.46_{\scriptscriptstyle\pm0.39}$}
& {\color[HTML]{656565} $0.39_{\scriptscriptstyle\pm0.09}$} & {\color[HTML]{656565} $1.16_{\scriptscriptstyle\pm0.19}$}
& {\color[HTML]{656565} $0.76_{\scriptscriptstyle\pm0.01}$} & {\color[HTML]{656565} $3.91_{\scriptscriptstyle\pm0.35}$}
& {\color[HTML]{656565} $-0.47_{\scriptscriptstyle\pm0.00}$} & {\color[HTML]{656565} $5.03_{\scriptscriptstyle\pm0.00}$}
& {\color[HTML]{656565} $0.22_{\scriptscriptstyle\pm0.00}$} & {\color[HTML]{656565} $1.27_{\scriptscriptstyle\pm0.12}$}
& {\color[HTML]{656565} $0.59_{\scriptscriptstyle\pm0.00}$} & {\color[HTML]{656565} $3.52_{\scriptscriptstyle\pm0.00}$}
& {\color[HTML]{0000FF} \boldmath\textbf{$0.48_{\scriptscriptstyle\pm0.33}$}} & {\color[HTML]{0000FF} \boldmath\textbf{$0.86_{\scriptscriptstyle\pm0.82}$}} \\
 
CarRun
& {\boldmath\textbf{$0.97_{\scriptscriptstyle\pm0.02}$}} & {\boldmath\textbf{$0.33_{\scriptscriptstyle\pm0.05}$}}
& {\boldmath\textbf{$0.94_{\scriptscriptstyle\pm0.00}$}} & {\boldmath\textbf{$0.22_{\scriptscriptstyle\pm0.02}$}}
& {\color[HTML]{0000FF} \boldmath\textbf{$0.99_{\scriptscriptstyle\pm0.01}$}} & {\color[HTML]{0000FF} \boldmath\textbf{$0.65_{\scriptscriptstyle\pm0.31}$}}
& {\boldmath\textbf{$0.94_{\scriptscriptstyle\pm0.01}$}} & {\boldmath\textbf{$0.15_{\scriptscriptstyle\pm0.91}$}}
& {\color[HTML]{656565} $0.68_{\scriptscriptstyle\pm0.01}$} & {\color[HTML]{656565} $7.78_{\scriptscriptstyle\pm0.09}$}
& {\color[HTML]{656565} $0.95_{\scriptscriptstyle\pm0.01}$} & {\color[HTML]{656565} $1.79_{\scriptscriptstyle\pm0.18}$}
& {\boldmath\textbf{$0.87_{\scriptscriptstyle\pm0.00}$}} & {\boldmath\textbf{$0.00_{\scriptscriptstyle\pm0.00}$}}
& {\boldmath\textbf{$0.80_{\scriptscriptstyle\pm0.08}$}} & {\boldmath\textbf{$0.11_{\scriptscriptstyle\pm0.15}$}} \\
 
DroneRun
& {\color[HTML]{656565} $0.24_{\scriptscriptstyle\pm0.02}$} & {\color[HTML]{656565} $2.13_{\scriptscriptstyle\pm0.62}$}
& {\boldmath\textbf{$0.28_{\scriptscriptstyle\pm0.25}$}} & {\boldmath\textbf{$0.74_{\scriptscriptstyle\pm0.97}$}}
& {\color[HTML]{0000FF} \boldmath\textbf{$0.63_{\scriptscriptstyle\pm0.04}$}} & {\color[HTML]{0000FF} \boldmath\textbf{$0.79_{\scriptscriptstyle\pm0.68}$}}
& {\color[HTML]{656565} $0.72_{\scriptscriptstyle\pm0.12}$} & {\color[HTML]{656565} $5.54_{\scriptscriptstyle\pm0.81}$}
& {\color[HTML]{656565} $0.42_{\scriptscriptstyle\pm0.10}$} & {\color[HTML]{656565} $2.47_{\scriptscriptstyle\pm0.34}$}
& {\color[HTML]{656565} $0.33_{\scriptscriptstyle\pm0.10}$} & {\color[HTML]{656565} $3.52_{\scriptscriptstyle\pm0.58}$}
& {\color[HTML]{656565} $0.67_{\scriptscriptstyle\pm0.02}$} & {\color[HTML]{656565} $4.15_{\scriptscriptstyle\pm0.10}$}
& {\boldmath\textbf{$0.10_{\scriptscriptstyle\pm0.15}$}} & {\boldmath\textbf{$0.84_{\scriptscriptstyle\pm1.01}$}} \\
 
AntRun
& {\color[HTML]{656565} $0.72_{\scriptscriptstyle\pm0.06}$} & {\color[HTML]{656565} $2.93_{\scriptscriptstyle\pm2.40}$}
& {\color[HTML]{656565} $0.65_{\scriptscriptstyle\pm0.15}$} & {\color[HTML]{656565} $1.09_{\scriptscriptstyle\pm0.84}$}
& {\color[HTML]{0000FF} \boldmath\textbf{$0.72_{\scriptscriptstyle\pm0.04}$}} & {\color[HTML]{0000FF} \boldmath\textbf{$0.91_{\scriptscriptstyle\pm0.42}$}}
& {\color[HTML]{656565} $0.76_{\scriptscriptstyle\pm0.07}$} & {\color[HTML]{656565} $5.11_{\scriptscriptstyle\pm2.39}$}
& {\boldmath\textbf{$0.15_{\scriptscriptstyle\pm0.02}$}} & {\boldmath\textbf{$0.73_{\scriptscriptstyle\pm0.07}$}}
& {\boldmath\textbf{$0.03_{\scriptscriptstyle\pm0.02}$}} & {\boldmath\textbf{$0.02_{\scriptscriptstyle\pm0.09}$}}
& {\boldmath\textbf{$0.61_{\scriptscriptstyle\pm0.01}$}} & {\boldmath\textbf{$0.94_{\scriptscriptstyle\pm0.69}$}}
& {\boldmath\textbf{$0.03_{\scriptscriptstyle\pm0.02}$}} & {\boldmath\textbf{$0.00_{\scriptscriptstyle\pm0.00}$}} \\
 
\midrule
 
BallCircle
& {\color[HTML]{656565} $0.74_{\scriptscriptstyle\pm0.15}$} & {\color[HTML]{656565} $4.71_{\scriptscriptstyle\pm1.79}$}
& {\boldmath\textbf{$0.52_{\scriptscriptstyle\pm0.08}$}} & {\boldmath\textbf{$0.65_{\scriptscriptstyle\pm0.17}$}}
& {\color[HTML]{656565} $0.77_{\scriptscriptstyle\pm0.06}$} & {\color[HTML]{656565} $1.07_{\scriptscriptstyle\pm0.27}$}
& {\color[HTML]{656565} $0.69_{\scriptscriptstyle\pm0.11}$} & {\color[HTML]{656565} $2.36_{\scriptscriptstyle\pm1.04}$}
& {\color[HTML]{656565} $0.86_{\scriptscriptstyle\pm0.18}$} & {\color[HTML]{656565} $3.09_{\scriptscriptstyle\pm1.53}$}
& {\color[HTML]{0000FF} \boldmath\textbf{$0.64_{\scriptscriptstyle\pm0.01}$}} & {\color[HTML]{0000FF} \boldmath\textbf{$0.76_{\scriptscriptstyle\pm0.00}$}}
& {\color[HTML]{656565} $0.70_{\scriptscriptstyle\pm0.04}$} & {\color[HTML]{656565} $2.61_{\scriptscriptstyle\pm0.79}$}
& {\boldmath\textbf{$0.61_{\scriptscriptstyle\pm0.15}$}} & {\boldmath\textbf{$0.21_{\scriptscriptstyle\pm0.31}$}} \\
 
CarCircle
& {\color[HTML]{656565} $0.58_{\scriptscriptstyle\pm0.25}$} & {\color[HTML]{656565} $3.74_{\scriptscriptstyle\pm2.20}$}
& {\boldmath\textbf{$0.50_{\scriptscriptstyle\pm0.22}$}} & {\boldmath\textbf{$0.84_{\scriptscriptstyle\pm0.67}$}}
& {\color[HTML]{0000FF} \boldmath\textbf{$0.75_{\scriptscriptstyle\pm0.06}$}} & {\color[HTML]{0000FF} \boldmath\textbf{$0.95_{\scriptscriptstyle\pm0.61}$}}
& {\color[HTML]{656565} $0.63_{\scriptscriptstyle\pm0.19}$} & {\color[HTML]{656565} $1.89_{\scriptscriptstyle\pm1.37}$}
& {\color[HTML]{656565} $0.74_{\scriptscriptstyle\pm0.10}$} & {\color[HTML]{656565} $2.18_{\scriptscriptstyle\pm1.33}$}
& {\boldmath\textbf{$0.71_{\scriptscriptstyle\pm0.02}$}} & {\boldmath\textbf{$0.33_{\scriptscriptstyle\pm0.00}$}}
& {\color[HTML]{656565} $0.49_{\scriptscriptstyle\pm0.05}$} & {\color[HTML]{656565} $3.14_{\scriptscriptstyle\pm2.98}$}
& {\boldmath\textbf{$0.63_{\scriptscriptstyle\pm0.05}$}} & {\boldmath\textbf{$0.11_{\scriptscriptstyle\pm0.21}$}} \\
 
DroneCircle
& {\color[HTML]{656565} $0.72_{\scriptscriptstyle\pm0.04}$} & {\color[HTML]{656565} $3.03_{\scriptscriptstyle\pm0.29}$}
& {\boldmath\textbf{$0.56_{\scriptscriptstyle\pm0.18}$}} & {\boldmath\textbf{$0.57_{\scriptscriptstyle\pm0.27}$}}
& {\color[HTML]{0000FF} \boldmath\textbf{$0.63_{\scriptscriptstyle\pm0.07}$}} & {\color[HTML]{0000FF} \boldmath\textbf{$0.98_{\scriptscriptstyle\pm0.27}$}}
& {\color[HTML]{656565} $0.80_{\scriptscriptstyle\pm0.07}$} & {\color[HTML]{656565} $3.07_{\scriptscriptstyle\pm0.89}$}
& {\color[HTML]{656565} $0.78_{\scriptscriptstyle\pm0.04}$} & {\color[HTML]{656565} $3.68_{\scriptscriptstyle\pm0.44}$}
& {\color[HTML]{656565} $-0.22_{\scriptscriptstyle\pm0.05}$} & {\color[HTML]{656565} $1.28_{\scriptscriptstyle\pm0.97}$}
& {\color[HTML]{656565} $0.26_{\scriptscriptstyle\pm0.03}$} & {\color[HTML]{656565} $1.02_{\scriptscriptstyle\pm0.46}$}
& {\boldmath\textbf{$-0.26_{\scriptscriptstyle\pm0.00}$}} & {\boldmath\textbf{$0.26_{\scriptscriptstyle\pm0.39}$}} \\
 
AntCircle
& {\color[HTML]{656565} $0.58_{\scriptscriptstyle\pm0.19}$} & {\color[HTML]{656565} $4.90_{\scriptscriptstyle\pm3.55}$}
& {\color[HTML]{0000FF} \boldmath\textbf{$0.40_{\scriptscriptstyle\pm0.16}$}} & {\color[HTML]{0000FF} \boldmath\textbf{$0.96_{\scriptscriptstyle\pm2.67}$}}
& {\color[HTML]{656565} $0.54_{\scriptscriptstyle\pm0.20}$} & {\color[HTML]{656565} $1.78_{\scriptscriptstyle\pm4.33}$}
& {\color[HTML]{656565} $0.58_{\scriptscriptstyle\pm0.25}$} & {\color[HTML]{656565} $2.87_{\scriptscriptstyle\pm3.08}$}
& {\color[HTML]{656565} $0.65_{\scriptscriptstyle\pm0.20}$} & {\color[HTML]{656565} $5.48_{\scriptscriptstyle\pm3.33}$}
& {\boldmath\textbf{$0.00_{\scriptscriptstyle\pm0.00}$}} & {\boldmath\textbf{$0.00_{\scriptscriptstyle\pm0.00}$}}
& {\color[HTML]{656565} $0.17_{\scriptscriptstyle\pm0.10}$} & {\color[HTML]{656565} $5.04_{\scriptscriptstyle\pm6.74}$}
& {\boldmath\textbf{$0.00_{\scriptscriptstyle\pm0.00}$}} & {\boldmath\textbf{$0.00_{\scriptscriptstyle\pm0.00}$}} \\
 
\midrule
 
\textbf{Average}
& {\color[HTML]{656565} $0.64_{\scriptscriptstyle\pm0.25}$} & {\color[HTML]{656565} $3.36_{\scriptscriptstyle\pm3.31}$}
& {\color[HTML]{0000FF} \boldmath\textbf{$0.52_{\scriptscriptstyle\pm0.27}$}} & {\color[HTML]{0000FF} \boldmath\textbf{$0.82_{\scriptscriptstyle\pm1.27}$}}
& {\color[HTML]{656565} $0.68_{\scriptscriptstyle\pm0.19}$} & {\color[HTML]{656565} $1.04_{\scriptscriptstyle\pm1.65}$}
& {\color[HTML]{656565} $0.74_{\scriptscriptstyle\pm0.25}$} & {\color[HTML]{656565} $3.11_{\scriptscriptstyle\pm3.55}$}
& {\color[HTML]{656565} $0.48_{\scriptscriptstyle\pm0.27}$} & {\color[HTML]{656565} $3.80_{\scriptscriptstyle\pm3.95}$}
& {\color[HTML]{656565} $0.33_{\scriptscriptstyle\pm0.29}$} & {\color[HTML]{656565} $1.12_{\scriptscriptstyle\pm1.85}$}
& {\color[HTML]{656565} $0.55_{\scriptscriptstyle\pm0.24}$} & {\color[HTML]{656565} $2.55_{\scriptscriptstyle\pm3.62}$}
& {\boldmath\textbf{$0.30_{\scriptscriptstyle\pm0.38}$}} & {\boldmath\textbf{$0.30_{\scriptscriptstyle\pm0.60}$}} \\
 
\bottomrule
\end{tabular}
}
\end{center}
\end{table}

\ref{tab:bulletgym-result} shows that PDOCRL is the only algorithm that consistently satisfies the safety constraint across all BulletGym tasks, with normalized cost remaining below 1 in every case.
At the same time, PDOCRL achieves strong reward performance, remaining competitive with or outperforming existing offline constrained RL baselines on several tasks.
\subsection{Ablation study}
 
We conduct an ablation study to demonstrate the effectiveness of the reparameterization trick used in our algorithm, which avoids the policy extraction step and eliminates the need for knowledge of the data-generating distribution. The ablated variant obtains $\widehat{w}$ by optimizing for the estimated Lagrangian \eqref{eqn:ablation}, then using the policy extraction step \eqref{eqn:policy-extraction} to obtain a policy.
 
As shown in \ref{tab:ablation-bulletgym}, PDOCRL shows better average reward and cost performance than the variant that requires policy extraction, suggesting that the policy extraction step can be a meaningful bottleneck in offline constrained reinforcement learning.
 
\begin{table}[htb]
\caption{\small Ablation study on BulletGym tasks. We compare PDOCRL (Ours) with an ablated variant, \textit{Policy extraction required}. $\uparrow$: higher reward is better. $\downarrow$: lower cost is better.}
\vspace{-2pt}
\label{tab:ablation-bulletgym}
\centering
\small
\begin{tabular}{c rr rr}
\toprule
\multirow{2}{*}{Task}
& \multicolumn{2}{c}{\textbf{PDOCRL (Ours)}}
& \multicolumn{2}{c}{Policy extraction required} \\
\cmidrule(lr){2-3} \cmidrule(lr){4-5}
& Reward\,$\uparrow$ & Cost\,$\downarrow$
& Reward\,$\uparrow$ & Cost\,$\downarrow$ \\
\midrule
 
BallRun
& $0.48_{\scriptscriptstyle\pm0.33}$ & $0.86_{\scriptscriptstyle\pm0.82}$
& $-1.06_{\scriptscriptstyle\pm0.15}$ & $2.69_{\scriptscriptstyle\pm1.44}$ \\

CarRun
& $0.80_{\scriptscriptstyle\pm0.08}$ & $0.11_{\scriptscriptstyle\pm0.15}$
& $-0.59_{\scriptscriptstyle\pm0.26}$ & $4.63_{\scriptscriptstyle\pm2.51}$ \\

DroneRun
& $0.10_{\scriptscriptstyle\pm0.15}$ & $0.84_{\scriptscriptstyle\pm1.01}$
& $-0.26_{\scriptscriptstyle\pm0.09}$ & $0.35_{\scriptscriptstyle\pm1.00}$ \\

AntRun
& $0.03_{\scriptscriptstyle\pm0.02}$ & $0.00_{\scriptscriptstyle\pm0.00}$
& $0.21_{\scriptscriptstyle\pm0.06}$ & $0.01_{\scriptscriptstyle\pm0.02}$ \\
 
\midrule
 
BallCircle
& $0.61_{\scriptscriptstyle\pm0.15}$ & $0.21_{\scriptscriptstyle\pm0.31}$
& $0.01_{\scriptscriptstyle\pm0.04}$ & $2.19_{\scriptscriptstyle\pm2.92}$ \\

CarCircle
& $0.63_{\scriptscriptstyle\pm0.05}$ & $0.11_{\scriptscriptstyle\pm0.21}$
& $-0.04_{\scriptscriptstyle\pm0.05}$ & $3.87_{\scriptscriptstyle\pm4.65}$ \\

DroneCircle
& $-0.26_{\scriptscriptstyle\pm0.00}$ & $0.26_{\scriptscriptstyle\pm0.39}$
& $-0.11_{\scriptscriptstyle\pm0.13}$ & $1.30_{\scriptscriptstyle\pm0.60}$ \\

AntCircle
& $0.00_{\scriptscriptstyle\pm0.00}$ & $0.00_{\scriptscriptstyle\pm0.00}$
& $0.06_{\scriptscriptstyle\pm0.02}$ & $0.65_{\scriptscriptstyle\pm0.35}$ \\
 
\midrule
 
\textbf{Average}
& $0.30_{\scriptscriptstyle\pm0.38}$ & $0.30_{\scriptscriptstyle\pm0.60}$
& $-0.22_{\scriptscriptstyle\pm0.41}$ & $1.96_{\scriptscriptstyle\pm2.74}$ \\
 
\bottomrule
\end{tabular}
\end{table}
\section{Conclusion}

We proposed PDOCRL, an oracle-efficient algorithm for offline constrained RL with general function approximation under partial data coverage.
Our analysis shows that restricted Lagrangian saddle-point formulations can admit spurious solutions, and that Slater's condition together with all-policy action-value realizability rules out this pathology.
Building on this guarantee, PDOCRL combines Lagrangian decomposition with an occupancy-ratio reparameterization to optimize policies directly, without access to the data-generating distribution.
A natural direction for future work is to relax the all-policy realizability condition, ideally toward a single-policy realizability assumption.
\section*{References}
\printbibliography[heading=none]

\newpage

\appendix

\section{Proof of Proposition~\ref{prop:policy-extraction}} \label{sec:policy-extraction}

We construct a single one-step contextual bandit ($\gamma = 0$), data distribution, and policy
class that work for all $n$. Only the learned ratio $\widehat w_n$ depends on
$n$.
Let $\mathcal S=\{c,r\}$, $\mathcal A=\{0,1\}$, $d_0(c) = d_0(r) = \frac{1}{2}$.
The behavior policy $\pi_B$ is uniform such that the induced occupancy measure is $\mu_D(s, a) = \frac{1}{4}$ for every $(s, a) \in \mathcal{S} \times \mathcal{A}$.
The reward function is $r(c,0)=r(c,1)=0$ and $r(r,1)=1, R(r,0)=0$.
Therefore, for any policy $\pi$,
$J(\pi) = \mathbb E_\pi[r(s,a)] = \frac12 \pi(1\mid r)$ and it follows that any policy $\pi$ with $\pi(1 | r) = 1$ is optimal with optimal value $J(\pi) = \frac{1}{2}$.

Let $\Pi = \{ \pi^\ast, \widetilde\pi \}$ where $\pi^\star$ is an optimal policy with $\pi^\star(0\mid c)=\frac18$, $\pi^\star(1\mid c)=\frac78$,  $\pi^\star(1\mid r)=1$, $\pi^\star(0\mid r)=0$;
and $\widetilde{\pi}$ is $\widetilde{\pi}(0\mid c)=\frac78$, $\widetilde\pi(1\mid c)=\frac18$, $\widetilde\pi(1\mid r)=\frac17$, $\widetilde\pi(0\mid r)=\frac67$.
Then, $J(\pi^\ast) = \frac{1}{2}$ and $J(\widetilde\pi) = \frac{1}{14}$.

We first verify boundedness of the true optimal density ratio.
Since this is a one-step contextual bandit, $\mu^{\pi^\star}(s,a) = \mathbb P(S=s)\pi^\star(a\mid s) = \frac12\pi^\star(a\mid s)$.
Because $\mu_D(s,a)=1/4$, we have $w^\star(s,a) = \frac{\mu^{\pi^\star}(s,a)}{\mu_D(s,a)} = 2\pi^\star(a\mid s)$.
Therefore, $w^\star(c,0)=\frac14$, $w^\star(c,1)=\frac74$, and $w^\star(r,1)=2$, $w^\star(r,0)=0$.
Thus $0\le w^\star(s,a)\le 2$ for all $(s,a)$.

Now, consider a learned density-ratio $\widehat{w}_n$ with $\widehat w_n(c,0)=2-\frac1n$, $\widehat w_n(c,1)=\frac1n$, $\widehat w_n(r,1)=2$ and $\widehat w_n(r,0)=0$.
Then, $0\le \widehat w_n(s,a)\le 2$ for all $(s, a)$.
Define $\widehat\mu_n(s,a):=\mu_D(s,a)\widehat w_n(s,a)$.
Since $\mu_D(s,a)=1/4$, we have
$\widehat\mu_n(c,0) = \frac14\left(2-\frac1n\right)$, $\widehat\mu_n(c,1) = \frac1{4n}$, $\widehat\mu_n(r,1) = \frac12$, and $\widehat\mu_n(r,0)=0$.
Hence, the policy extracted from $\widehat\mu$ puts weight 1 on state $r$, and is optimal, showing the first claim.

Next we analyze the population weighted-log extraction over the finite class $\Pi$:
$$
\argmax_{\pi \in \Pi} L_n(\pi) \coloneqq \mathbb{E}_{\mu_D}[ \widehat{w}_n \log \pi(a | s)].
$$
It can be shown that
$$
L_n(\pi^\ast) - L_n (\widetilde\pi) = \frac{\log 7}{2n} > 0,
$$
Hence the population weighted-log extraction over $\Pi$ uniquely selects $\pi^\star$, showing the second claim.

It remains to show that the empirical weighted-log extraction from the dataset $\mathcal{D}$
$$
\argmax_{\pi \in \Pi} \widehat{L}_n(\pi) \coloneqq \frac{1}{n} \sum_{i = 1}^n \widehat{w}_n(s_i, a_i) \log \pi(a_i | s_i),
$$
selects the bad policy with constant probability.
Define the counts
\[
    N_{s, a} := \sum_{i=1}^n \mathbb{I}\{(s_i,a_i)=(s,a)\}.
\]
Then,
\[
\begin{aligned}
\widehat L_n(\pi^\star)-\widehat L_n(\widetilde\pi)
&= \frac1n \sum_{i=1}^n \widehat w_n(S_i,A_i) \log\frac{\pi^\star(A_i\mid S_i)}{\widetilde\pi(A_i\mid S_i)} \\
&= \frac{\log 7}{n} \left[ -2 (N_{c0} - N_{r1}) + \frac{1}{n} (N_{c0} + N_{c1}) \right] \\
&\leq \frac{\log 7}{n} \left[ -2 (N_{c0} - N_{r1}) + 1 \right],
\end{aligned}
\]
where the inequality follows since $N_{c0} - N_{r1} \leq n$.
Hence, $\widehat{L}_n (\pi^\ast) < \widehat{L}_n(\widetilde\pi)$ and the empirical policy extraction procedure will select the suboptimal $\widehat\pi$, as long as $N_{c0} - N_{r1} \geq 1$.

Now, we lower-bound $\mathbb P(N_{c0} - N_{r1} \geq 1)$. For each $i$, define
$$
X_i := \mathbb{I}\{(s_i,a_i)=(c,0)\} - \mathbb{I}\{(s_i,a_i)=(r,1)\}.
$$
Because $\mu_D$ is uniform over the four state-action pairs, we have
\[
    \mathbb P(X_i=1)=\frac14,
    \qquad
    \mathbb P(X_i=-1)=\frac14,
    \qquad
    \mathbb P(X_i=0)=\frac12.
\]
Let $X:=\sum_{i=1}^n X_i = N_{c0} - N_{r1}$.
The distribution of $X$ is symmetric around zero, because $X_i$ and $-X_i$ have the same distribution.
Therefore, $\mathbb P(X>0)=\mathbb P(X<0)$ and it follows that
$$
\mathbb P(X>0) = \frac{1-\mathbb P(X=0)}{2}.
$$
We now prove that $\mathbb P(X=0)\le 1/2$. Let $K:=\sum_{i=1}^n \mathbf 1\{|X_i|=1\}$.
Then $K\sim \operatorname{Binomial}\left(n,\frac12\right)$.
Conditional on $K=k$, the $k$ nonzero signs are i.i.d. Rademacher random
variables. If $k=0$, then $X=0$ with probability $1$. If $k=1$, then
$X\neq 0$ with probability $1$. If $k\ge 2$, then
\[
    \mathbb P(X=0\mid K=k)\le \frac12.
\]
Indeed, for odd $k$, the conditional probability is $0$. For even $k\ge 2$,
\[
    \mathbb P(X=0\mid K=k)
    =
    \binom{k}{k/2}2^{-k}.
\]
For $k=2$, this equals $1/2$, and for every even $k\ge 4$, the central
binomial probability is strictly smaller than $1/2$.
Hence
\[
\begin{aligned}
\mathbb P(X=0)
&= \sum_{k=0}^n \mathbb P(X=0\mid K=k)\mathbb P(K=k) \\
&\le \mathbb P(K=0) + \frac12\mathbb P(K\ge 2) \\
&= 2^{-n} + \frac12 \left( 1-2^{-n}-n2^{-n} \right) \\
&= \frac12+\frac{1-n}{2^{n+1}} \le \frac12,
\end{aligned}
\]
where the second equality uses $\mathbb P(K=0)=2^{-n}$ and $\mathbb P(K=1)=n2^{-n}$.
Consequently,
\[
\mathbb P(N_{c0} - N_{r1} > 0) = \mathbb P(X>0) = \frac{1-\mathbb P(X=0)}{2} \ge \frac14.
\]
Hence, the suboptimal policy $\widetilde\pi$ is chosen with probability at least $1/4$, completing the proof.

\section{Known data distribution case} \label{appendix:known-data-distribution}

Introducing Lagrangian multipliers $\bm{V} \in \mathbb{R}^{\vert \mathcal{S} \vert}$ for the Bellman flow constraints and \(\bm\lambda \in \mathbb{R}_+^I\) for the safety constraints in the linear program, the Lagrangian function is
$$
L(\bm\mu, \bm{V}, \bm\lambda) \coloneqq \langle \bm\mu, \bm{r} \rangle + \langle \bm{V}, (1 - \gamma) \bm{d}_0 + \gamma \bm{P}^\top \bm\mu - \bm{E}^\top \bm\mu \rangle + \sum_{i=1}^I \lambda_i (\langle \bm\mu, \bm{r}_i \rangle - \tau_i),
$$
In the general function approximation setting, the restricted saddle-point problem over $\mathcal{F}(\mathcal{U}) \times (\mathcal{V} \times \mathbb{R}_+^I)$ can admit spurious saddle points that do not correspond to optimal policies.
The below proposition shows that merely realizing an optimal policy, its occupancy measure, and its associated action-value functions is not sufficient to rule this out.

\begin{proposition} \label{prop:negative}
There exists a CMDP $(\mathcal{S}, \mathcal{A}, r_0, \{ r_i \}_{i = 1}^I, \gamma, d_0)$ and function classes $\mathcal{U}\subseteq \mathbb{R}_+^{\mathcal{S} \times \mathcal{A}}$ and $\mathcal{V} \subseteq \mathbb{R}^{\mathcal{S}}$ with $\mu^\ast \in \mathcal{U}$ and $V^\ast \in \mathcal{V}$ with the following property.
The function classes satisfy realizability assumption: $\mu^\ast \in \mathcal{U}$, $V_i^\ast \in \mathcal{V}$ for all $i = 0, 1, \dots, I$. However, there exists a saddle point $(\widehat\mu ; \widehat{V}, \widehat\lambda)$ of the Lagrangian function $L$ associated with the linear programming formulation of the MDP over $\mathcal{U} \times (\mathcal{V} \times \mathbb{R}_+^I)$, where the extracted policy $\widehat\pi$ is not optimal.
\end{proposition}

\begin{proof}[Proof of Proposition~\ref{prop:negative}]
We explicitly construct a CMDP
$(\mathcal{S}, \mathcal{A}, r_0, \{r_i\}_{i=1}^I, \gamma, d_0)$ and
function classes $\mathcal{U}$, $\mathcal{V}$ as follows; see
Figure~\ref{fig:cmdp-negative} for an illustration.

For notational simplicity, we identify the restricted primal class
$\mathcal{U}$ with the set of admissible occupancy measures,
and construct it so that
\[
\mathcal{U} = \{ \mu^\ast, \widetilde{\mu} \}.
\]
We explicitly construct a CMDP
$(\mathcal{S}, \mathcal{A}, r_0, \{r_i\}_{i=1}^I, \gamma, d_0)$ and
function classes $\mathcal{U}$, $\mathcal{V}$ as follows.

Let $I = 1$, $\mathcal{S} = \{ s_0, l_1, r_1, l_2, r_2 \}$, and
$\mathcal{A} = \{ L, R \}$.
Let the objective reward $r_0$ be defined by
\[
r_0(l_1,a)=1,\qquad r_0(l_2,a)=4
\qquad \text{for all } a\in\mathcal{A},
\]
and all other entries of $r_0$ are zero.
We choose the constraint reward to be identical:
\[
r_1 := r_0, \qquad \tau_1 := 0.
\]
Let $\gamma = 1/2$ and $d_0(s_0)=1$, so that the process starts from
$s_0$ deterministically.

The states $l_1$, $l_2$, and $r_2$ are absorbing, i.e.,
\[
P(s\mid s,L)=P(s\mid s,R)=1
\qquad \text{for } s\in\{l_1,l_2,r_2\}.
\]
At state $s_0$, action $L$ transitions uniformly to the next layer:
\[
P(l_1\mid s_0,L)=P(r_1\mid s_0,L)=1/2.
\]
At state $r_1$, action $L$ also transitions uniformly:
\[
P(l_2\mid r_1,L)=P(r_2\mid r_1,L)=1/2.
\]
On the other hand, action $R$ is biased to the right:
\[
P(l_1\mid s_0,R)=P(l_2\mid r_1,R)=1/4,
\qquad
P(r_1\mid s_0,R)=P(r_2\mid r_1,R)=3/4.
\]

Since $r_1=r_0\ge 0$ and $\mu\ge 0$, the safety constraint
\[
\langle \mu, r_1\rangle \ge \tau_1 = 0
\]
is redundant for every nonnegative $\mu$.
Hence the set of optimal policies of this CMDP coincides with that of the
underlying unconstrained MDP with reward $r_0$.

Let $V^\ast := V_0^\ast = V_1^\ast$.
A direct calculation shows that
\[
V^\ast(s_0)=1,\qquad V^\ast(l_1)=2,\qquad V^\ast(r_1)=2,\qquad
V^\ast(l_2)=8,\qquad V^\ast(r_2)=0.
\]
Moreover, an optimal occupancy measure is induced by the policy that always
chooses action $L$:
\[
\mu^\ast(s_0,L)=\frac12,\qquad
\mu^\ast(l_1,L)=\frac14,\qquad
\mu^\ast(r_1,L)=\frac18,\qquad
\mu^\ast(l_2,L)=\mu^\ast(r_2,L)=\frac1{16},
\]
and all other entries are zero.

The CMDP Lagrangian is
\[
L(\mu;V,\lambda)
=
\langle \mu,r_0\rangle
+
\langle V,(1-\gamma)d_0+\gamma P^\top\mu-E^\top\mu\rangle
+
\lambda(\langle \mu,r_1\rangle-\tau_1).
\]
At $(V^\ast,0)$, this becomes
\[
L(\mu;V^\ast,0)
=
\langle \mu,r_0\rangle
+
\langle V^\ast,(1-\gamma)d_0+\gamma P^\top\mu-E^\top\mu\rangle.
\]
Using Bellman optimality of $V^\ast$, every action has zero Bellman residual
except $(r_1,R)$, for which
\[
r_0(r_1,R)+\gamma \sum_{s'} P(s'\mid r_1,R)V^\ast(s')-V^\ast(r_1)
=
0+\frac12\Bigl(\frac14\cdot 8+\frac34\cdot 0\Bigr)-2
=
-1.
\]
Therefore,
\[
L(\mu;V^\ast,0)
=
(1-\gamma)V^\ast(s_0)-\mu(r_1,R)
=
\frac12 V^\ast(s_0)-\mu(r_1,R).
\]
Hence $L(\mu;V^\ast,0)$ depends on $\mu$ only through $\mu(r_1,R)$, and any
$\mu$ with $\mu(r_1,R)=0$ is a best response to $(V^\ast,0)$.

Now define $\widetilde{\mu}$ by
\[
\widetilde{\mu}(s_0,R)=\frac12,\qquad
\widetilde{\mu}(l_1,L)=\widetilde{\mu}(l_1,R)=\frac14,
\]
and all other entries zero.
Then $\widetilde{\mu}(r_1,R)=0$, so $\widetilde{\mu}$ is also a maximizer of
$L(\cdot;V^\ast,0)$ over $\mathcal{F}(\mathcal{U})$.
Let
\[
\mathcal{V}=\{V^\ast\}.
\]
Since $r_1=r_0\ge 0$ and $\tau_1=0$, for any $\lambda\ge 0$,
\[
L(\widetilde{\mu};V^\ast,\lambda)
=
L(\widetilde{\mu};V^\ast,0)+\lambda\langle \widetilde{\mu},r_0\rangle
\ge
L(\widetilde{\mu};V^\ast,0).
\]
Thus $\lambda=0$ is a minimizing dual variable for $\widetilde{\mu}$.
Because $\mathcal{V}$ is a singleton and $\widetilde{\mu}$ maximizes
$L(\cdot;V^\ast,0)$, we conclude that
\[
(\widetilde{\mu};V^\ast,0)
\]
is a saddle point of $L$ over
$\mathcal{U}\times(\mathcal{V}\times\mathbb{R}_+)$.

The realizability condition also holds. By construction,
$\mu^\ast\in\mathcal{U}$, and since $r_1=r_0$, we have
\[
V_1^\ast = V_0^\ast = V^\ast \in \mathcal{V}.
\]

Finally, we show that the policy extracted from $\widetilde{\mu}$ is not
optimal. Define the extracted policy by
\[
\pi_\mu(a\mid s)
=
\begin{cases}
\dfrac{\mu(s,a)}{\sum_{a'}\mu(s,a')},
& \text{if } \sum_{a'}\mu(s,a')>0,\\[1.2ex]
\mathbf{1}\{a=R\},
& \text{otherwise}.
\end{cases}
\]
Under this rule, the policy $\widetilde{\pi}:=\pi_{\widetilde{\mu}}$ chooses
action $R$ at $s_0$, and since
$\widetilde{\mu}(r_1,L)=\widetilde{\mu}(r_1,R)=0$, it also chooses the default
action $R$ at $r_1$.
Therefore,
\[
V_{r_0}^{\widetilde{\pi}}(r_1)
=
\frac12\Bigl(\frac14 V^\ast(l_2)+\frac34 V^\ast(r_2)\Bigr)
=
\frac12\Bigl(\frac14\cdot 8+\frac34\cdot 0\Bigr)
=
1,
\]
and hence
\[
V_{r_0}^{\widetilde{\pi}}(s_0)
=
\frac12\Bigl(\frac14 V^\ast(l_1)+\frac34 V_{r_0}^{\widetilde{\pi}}(r_1)\Bigr)
=
\frac12\Bigl(\frac14\cdot 2+\frac34\cdot 1\Bigr)
=
\frac58
<
1
=
V^\ast(s_0).
\]
Thus $\widetilde{\pi}$ is not optimal, even though
$(\widetilde{\mu};V^\ast,0)$ is a saddle point of the restricted CMDP
Lagrangian.
This proves the proposition.
\end{proof}

\begin{figure}[t]
\centering
\begin{tikzpicture}[
    state/.style={draw, circle, minimum size=0.9cm, line width=0.5pt, inner sep=0pt},
    info/.style={font=\scriptsize},
    arrow/.style={-{Latex}, line width=0.8pt},
    dashedarrow/.style={-{Latex}, dashed, line width=0.8pt}
]

\node[state] (S)  at (0,0) {$s_0$};
\node[state] (L1) at (-1.8,-1.6) {$l_1$};
\node[state] (R1) at (1.8,-1.6) {$r_1$};
\node[state] (L2) at (0.6,-3.2) {$l_2$};
\node[state] (R2) at (3.0,-3.2) {$r_2$};

\draw[arrow] (S)  to[out=225, in=45]  node[midway, above left] {$L$} (L1);
\draw[arrow] (S)  to[out=-45, in=135] node[midway, above right] {$R$} (R1);
\draw[arrow] (R1) to[out=-135, in=45] node[midway, left] {$L$} (L2);
\draw[dashedarrow] (R1) to[out=-45, in=135] node[midway, right] {$R$} (R2);

\draw[arrow] (L1) to[out=-135, in=-45, looseness=4] (L1);
\draw[arrow] (L2) to[out=-135, in=-45, looseness=4] (L2);
\draw[arrow] (R2) to[out=-135, in=-45, looseness=4] (R2);

\node[info] at (-3.45,-2.55) {$r_0=r_1=1$};
\node[info] at (-0.25,-4.35) {$r_0=r_1=4$};
\node[info] at (2.15,-4.35) {$r_0=r_1=0$};

\node[info] at (3.65,0.0) {$\tau_1=0$};

\end{tikzpicture}
\caption{
Counterexample for Proposition~\ref{prop:negative}. Here $r_1=r_0$ and $\tau_1=0$, so the constraint is redundant. Nevertheless, the restricted saddle point $(\widetilde{\mu};V^\ast,0)$ induces a suboptimal extracted policy.
}
\label{fig:cmdp-negative}
\end{figure}

This proposition suggests that finding a saddle point of $L$ over $\mathcal{U} \times (\mathcal{V} \times \mathcal{\lambda})$ may not give an optimal solution even when $\mathcal{U} \times (\mathcal{V} \times \mathcal{\lambda})$ contains an optimal solution.
As a workaround, \textcite{zhan2022offline} add a regularization term $- \alpha \mathbb{E}_\mu[f(\mu / \mu_D)]$ to enforce strong concavity of $L$ in terms of $\mu$.
However, this approach leads to sample complexity for finding $\varepsilon$ near optimal policy scales with $1 / \varepsilon^{4}$.
Rather than regularizing the saddle problem, we make explicit a stronger realizability requirement on the function class $\mathcal{V}$:

\begin{assumption}[All-Policy Value Function Realizability] \label{assumption:all-policy-state-value-realizability}
For every policy $\pi \in \Pi$, we have $V_0^\pi \in \mathcal{V}$ and $V_0^\pi + (1 + \frac{1}{\varphi}) V_i^\pi \in \mathcal{V}$ for all $i = 0, 1, \dots, I$.
\end{assumption}

Compared with prior LP-based analyses that only require realizability of an optimal value function~\parencite{zhan2022offline}, this is a stronger all-policy assumption.
It is precisely this strengthening that rules out spurious saddle points in the restricted saddle problem.

With this stronger assumption on $\mathcal{V}$, we can show that a saddle point of $L$ over $\mathcal{U} \times \mathcal{V}$ is optimal, as formally stated in the following:

\begin{lemma} \label{lemma:saddle-point}
Let $B \coloneqq1 + \frac{1}{\varphi} $.
Assume Assumption~\ref{assumption:slater} and ~\ref{assumption:all-policy-state-value-realizability}, and let \(\Lambda \coloneqq \{ \lambda \in \mathbb{R}_+^I : \|\lambda\|_1 \le  B\}\).
Consider function classes \(\mathcal{U} \subseteq \mathbb{R}_+^{\vert \mathcal{S} \times \mathcal{A} \vert}\) and
\(\mathcal{V} \subseteq \mathbb{R}^{\vert \mathcal{S}  \vert}\).
If \((\widehat\mu, ; \widehat V, \widehat\lambda)\) is a saddle point of \(L\) over
$\mathcal{U} \times (\mathcal{V} \times \Lambda)$, then \(\widehat\pi\) is an optimal feasible policy.
\end{lemma}
\begin{proof}
Let \(\bm{e}_i \in \mathbb{R}^I\) denote the \(i\)-th standard basis vector.
Using the reparameterized form of the Lagrangian, for any \(\mu \in \mathcal{U}\), the Bellman equations imply
\[
L(\mu; V_0^\pi, \bm{0}) = (1 - \gamma) J_0(\pi),
\]
and, since \(V_0^\pi + B V_i^\pi\) is the value function of \(\pi\) under reward \(r_0 + B r_i\),
\[
L(\mu; V_0^\pi + B V_i^\pi, B e_i)
=
(1 - \gamma) J_0(\pi) + B \bigl( (1 - \gamma) J_i(\pi) - \tau_i \bigr)
\]
for each \(i = 1, \dots, I\).

Next, since \(\mu^\ast = \mu^{\pi^\ast}\), for any \(V \in \mathcal{V}\) and \(\lambda \in \Lambda\),
\[
L(\mu^\ast; V, \lambda)
=
(1 - \gamma) J_0(\pi^\ast)
+
\sum_{i=1}^I \lambda_i \bigl( (1 - \gamma) J_i(\pi^\ast) - \tau_i \bigr).
\]
Since \(\pi^\ast\) is feasible, the second term is nonnegative, and hence
\[
L(\mu^\ast; V, \lambda)
\ge
(1 - \gamma) J_0(\pi^\ast).
\]

By the saddle-point property, choosing \((V,\lambda) = (V_0^{\widehat\pi}, 0)\) gives
\[
(1 - \gamma) J_0(\pi^\ast)
\le
L(\mu^\ast; \widehat V, \widehat\lambda)
\le
L(\widehat\mu, ; V_0^{\widehat\pi}, 0)
=
(1 - \gamma) J_0(\widehat\pi).
\]
Thus \(\widehat\pi\) attains objective value at least that of the optimal feasible policy \(\pi^\ast\).

It remains to show feasibility of \(\widehat\pi\).
Let \(\lambda^\ast\) be an optimal dual variable of the original CMDP.
By Lemma~\ref{lemma:dual-variable-bound}, \(\|\lambda^\ast\|_1 \le \frac{1}{\varphi} < B\).
Define
\[
m \coloneqq \min_{i=1,\dots,I} \bigl( (1 - \gamma) J_i(\widehat\pi) - \tau_i \bigr),
\]
and let \(j \in \arg\min_{i=1,\dots,I} \bigl( (1 - \gamma) J_i(\widehat\pi) - \tau_i \bigr)\).

By strong duality for the original CMDP and optimality of \(\lambda^\ast\), we have
\[
(1 - \gamma) J_0(\pi^\ast)
\ge
(1 - \gamma) J_0(\widehat\pi)
+
\sum_{i=1}^I \lambda_i^\ast \bigl( (1 - \gamma) J_i(\widehat\pi) - \tau_i \bigr)
\ge
(1 - \gamma) J_0(\widehat\pi) + m \|\lambda^\ast\|_1.
\]
On the other hand, applying the saddle-point property with
\((V,\lambda) = (V_0^{\widehat\pi} + B V_j^{\widehat\pi}, B \bm{e}_j)\) yields
\[
(1 - \gamma) J_0(\pi^\ast)
\le
L(\widehat\mu, ; V_0^{\widehat\pi} + B V_j^{\widehat\pi}, B \bm{e}_j)
=
(1 - \gamma) J_0(\widehat\pi) + B m.
\]
Combining the previous two displays, we obtain
\[
m \|\lambda^\ast\|_1 \le B m.
\]
If \(m < 0\), then dividing by \(m\) reverses the inequality and gives
\(\|\lambda^\ast\|_1 \ge B\), which contradicts \(\|\lambda^\ast\|_1 < B\).
Hence \(m \ge 0\), that is,
\[
(1 - \gamma) J_i(\widehat\pi) \ge \tau_i,
\qquad i = 1, \dots, I.
\]
So \(\widehat\pi\) is feasible.

We have shown that \(\widehat\pi\) is feasible and that
\(J_0(\widehat\pi) \ge J_0(\pi^\ast)\).
Since \(\pi^\ast\) is an optimal feasible policy, it follows that \(\widehat\pi\) is also optimal.
\end{proof}

With the result in the proposition above, we consider an algorithm that finds a saddle point $(\widehat\mu, \widehat{V}, \widehat{\lambda})$ that solves $\max_{\mu \in \mathcal{U}} \min_{V \in \mathcal{V}, \lambda \in \mathbb{R}_+^I} L(\mu, V, \lambda)$, then returns the policy $\widehat\pi = \pi(\widehat\mu)$ extracted from $\widehat\mu$. 
Since computing the Lagrangian function requires the knowledge of the transition probability kernel $P$, we need to estimate the Lagrangian function.
To facilitate the estimation, we change the measure from $\mu$ to $\mu_D$ by multiplying the importance weight $w = \mu / \mu_D$ and express the Lagrangian in terms of $w$ instead of $\mu$ as follows.
$$
L(\bm{w}, \bm{V}, \bm{\lambda}) \coloneqq (1 - \gamma) V(s_0) + \mathbb{E}_{\mu_D}[(w(s, a)(r(s, a) + \gamma [PV](s, a) - V(s))]+ \sum_{i=1}^I \lambda_i (\langle \bm\mu, \bm{r}_i \rangle - \tau_i).
$$
Analogous to Lemma~\ref{lemma:saddle-point}, we can show that if a function class $\mathcal{W}$ for the MIW contains the optimal $w^\ast = \mu^\ast / \mu_D$ and a function class $\mathcal{V}$ contains $V_0^\pi + (1 + \frac{1}{\varphi}) V_i^\pi$ for all policies $\pi \in \Pi$, then a saddle point of $L(w, V, \lambda)$ over $\mathcal{W} \times \mathcal{V} \times \Lambda$ is optimal.
In addition to Assumption~\ref{assumption:all-policy-state-value-realizability}, we make the standard optimal-density-ratio realizability assumption on $\mathcal{W}$ (Assumption~\ref{assumption:miw-realizability}).

With the realizability assumptions, we consider an algorithm that finds a saddle point of the estimated Lagrangian:
\begin{equation}\label{eqn:primal-dual}
\begin{aligned}
&\max_{w \in \mathcal{W}} \min_{V \in \mathcal{V},\, \lambda \in \Lambda}
\ \widehat{L}(w;V,\lambda) \\
&:= (1-\gamma)V(s_0)
+ \frac{1}{n}\sum_{j=1}^n w(s_j,a_j)
\Bigl(
r_{\bm\lambda}(s_j,a_j)
+ \gamma V(s_j') - V(s_j)
\Bigr)
- \sum_{i=1}^I \lambda_i \tau_i .
\end{aligned}
\end{equation}
where $r_{\bm\lambda}= r_0 +\sum_{i=1}^I\lambda_ir_i$, and $\widehat{L}(w, V, \lambda)$ is an unbiased estimate of $L(w, V, \lambda)$.
After finding a saddle point $(\widehat{w}, \widehat{V}, \widehat{\lambda})$, we need to extract policy from $\widehat{w}$.
In general, the policy extraction from an importance weight $w \in \mathcal{W}$ can be done by the policy extraction from the corresponding occupancy measure $\mu = w \cdot \mu_D$ and using \eqref{eqn:policy-extraction}:
\begin{equation} \label{eqn:policy-extraction-w}
[\pi(w)](a | s) = \frac{w(s, a) \mu_D(s, a)}{\sum_{a'} w(s, a') \mu_D(s, a')}
~~\text{if}~\sum_{a'} w(s, a') \mu_D(s, a') > 0,\quad \frac{1}{\vert \mathcal{A} \vert}~\text{otherwise}.
\end{equation}
The policy $\pi(\widehat{w})$ extracted from a solution $(\widehat{w}, \widehat{V}, \widehat{\lambda})$ of the saddle point problem \eqref{eqn:primal-dual} gives the guarantee. We show the guarantee below.

\begin{lemma} \label{lemma:concentration-state-value}
Consider the empirical Lagrangian estimate
\[
\widehat{L}(w; V, \lambda)
=
(1-\gamma)V(s_0)
+
\frac{1}{n}\sum_{j=1}^n
w(s_j,a_j)
\Bigl(
r_{\bm\lambda}(s_j,a_j)
+
\gamma V(s_j') - V(s_j)
\Bigr)
-
\sum_{i=1}^I \lambda_i \tau_i.
\]
where $r_{\bm\lambda}= r_0 +\sum_{i=1}^I\lambda_ir_i$. 
Given a function class $\mathcal{W}$ and $\mathcal{V}$ for $w$ and $V$, with boundedness condition $\|w\|_\infty \le C$ for all $w \in \mathcal{W}$, $\|V\|_\infty \le \frac{1}{1-\gamma} $ for all $V \in \mathcal{V}$, we have with probability at least $1 -\delta$ that, for all $w \in \mathcal{W}$, $V \in \mathcal{V}$ and $\lambda \in B \Delta^I$

\[
\bigl|L(w;V,\lambda)-\widehat{L}(w;V,\lambda)\bigr|
\le
\epsilon_n
\]

where 
$
\epsilon_n = 
\mathcal{O}\!\left(
\left(
1+\frac{C}{1-\gamma}(1+B)
\right)
\sqrt{
{\log\bigl((I+1) |\mathcal{W}||\mathcal{V}| / \delta\bigr)} /{n}
}
\right).
$
\end{lemma}

\begin{proof}
Fix $w \in \mathcal{W}$, $V \in \mathcal{V}$, and $\lambda \in \{ 0, B \bm{e}_1, \dots, B \bm{e}_I\}$.
Define
\[
Z_j
:=
(1-\gamma)V(s_0)
+
w(s_j,a_j)
\Bigl(
r_0(s_j,a_j)
+
\sum_{i=1}^I \lambda_i r_i(s_j,a_j)
+
\gamma V(s_j')-V(s_j)
\Bigr)
-
\sum_{i=1}^I \lambda_i\tau_i.
\]
Then
\[
\widehat{L}(w;V,\lambda)=\frac{1}{n}\sum_{j=1}^n Z_j,
\qquad
L(w;V,\lambda)=\mathbb{E}[Z_j].
\]

We now bound $|Z_j|$ uniformly.
Using $\|V\|_\infty \le \frac{1}{1-\gamma}$, we have
\[
|(1-\gamma)V(s_0)| \le 1.
\]
Next, since $\|w\|_\infty \le C$, $\|r_0\|_\infty \le 1$, $\|r_i\|_\infty \le 1$,
and $\|\lambda\|_1 \le B$,
\[
\left|
r_0(s_j,a_j)+\sum_{i=1}^I \lambda_i r_i(s_j,a_j)
\right|
\le
1+B.
\]
Also,
\[
|\gamma V(s_j')-V(s_j)|
\le
\gamma \|V\|_\infty + \|V\|_\infty
\le
\frac{1+\gamma}{1-\gamma}
\le
\frac{2}{1-\gamma}.
\]
Therefore,
\[
\left|
w(s_j,a_j)
\Bigl(
r_0(s_j,a_j)
+\sum_{i=1}^I \lambda_i r_i(s_j,a_j)
+\gamma V(s_j')-V(s_j)
\Bigr)
\right|
\le
C\left(1+B+\frac{2}{1-\gamma}\right).
\]
Finally,
\[
\left|\sum_{i=1}^I \lambda_i\tau_i\right|
\le
\|\lambda\|_1 \|\tau\|_\infty
\le
B\|\tau\|_\infty.
\]
Hence,
\[
|Z_j|
\le
1
+
C\left(1+B+\frac{2}{1-\gamma}\right)
+
B\|\tau\|_\infty
=
\mathcal{O}\!\left(
1+\frac{C}{1-\gamma}(1+B)
\right).
\]

Since $Z_1,\dots,Z_n$ are i.i.d. and uniformly bounded, Hoeffding's inequality
implies that for any fixed triple $(w,V,\lambda)$,
\[
\bigl|L(w;V,\lambda)-\widehat{L}(w;V,\lambda)\bigr|
=
\left|
\mathbb{E}[Z_j]-\frac{1}{n}\sum_{j=1}^n Z_j
\right|
\le
\mathcal{O}\!\left(
\left(
1+\frac{C}{1-\gamma}(1+B)
\right)
\sqrt{\frac{\log(1/\delta)}{n}}
\right)
\]
with probability at least $1-\delta$.

Applying a union bound over all triples in
$\mathcal{W}\times\mathcal{V}\times \{ 0, B \bm{e}_1, \dots, B \bm{e}_I\}$ yields that, with probability at least
$1-\delta$, the above bound holds uniformly for all
$w\in\mathcal{W}$, $V\in\mathcal{V}$, and $\lambda\in \{ 0, B \bm{e}_1, \dots, B \bm{e}_I\}$:
\[
\bigl|L(w;V,\lambda)-\widehat{L}(w;V,\lambda)\bigr|
\le
\mathcal{O}\!\left(
\left(
1+\frac{C}{1-\gamma}(1+B)
\right)
\sqrt{
\frac{\log\bigl((I+1) |\mathcal{W}||\mathcal{V}|/\delta\bigr)}{n}
}
\right).
\]
Now, for any $\lambda \in B \Delta_I$, we can express $\lambda = B \sum_{i = 1}^I \alpha_i \bm{e}_i$ for some $\alpha_0, \alpha_1, \dots, \alpha_I \geq 0$ with $\sum_{i = 0}^I \alpha_i = 1$.
Since $L(w, V, \lambda)$ and $\widehat{L}(w, V, \lambda)$ are affine in $\lambda$, we have
$$
L(w, V, \lambda) = \alpha_0 L(w, V, 0) + \sum_{i = 1}^I \alpha_i L(w, V, \bm{e}_i), \quad
\widehat{L}(w, V, \lambda) = \alpha_0 \widehat{L}(w, V, 0) + \sum_{i = 1}^I \alpha_i \widehat{L}(w, V, \bm{e}_i),
$$
and the uniform concentration inequality we obtained for $\lambda \in \{0, B \bm{e}_1, \dots, B\bm{e}_I\}$ gives concentration inequality for all $\lambda \in B \Delta^I$.
This completes the proof.
\end{proof}

\begin{theorem}\label{theorem:constrained-saddle-point}
Let $\mathcal{W}$ be a function class such that $\|w\|_\infty \le C$ and $\mathcal{V}$ be a function class such that $\|V\|_\infty \le \frac{1}{1-\gamma} $. 
Under Assumptions 
~\ref{assumption:slater},
\ref{assumption:miw-realizability},
and
\ref{assumption:all-policy-state-value-realizability},
if \((\widehat{w}; \widehat{V}, \widehat{\lambda})\) is a saddle point that solves
\[
\max_{w \in \mathcal{W}} \min_{V \in \mathcal{V},\, \lambda \in (1 + \frac{1}{\varphi}) \Delta^I}
\widehat{L}(w; V, \lambda),
\]
then the extracted policy \(\widehat\pi = \pi(\widehat{w})\) satisfies, with
probability at least \(1-\delta\),
\[
J_0(\widehat\pi) \ge J_0(\pi^\ast) - \frac{2\varepsilon_n}{1-\gamma},
\qquad
(1-\gamma)J_i(\widehat\pi) \ge \tau_i - \frac{4\varepsilon_n}{B},
\qquad i=1,\dots,I,
\]
where
\[
\varepsilon_n
=
\mathcal{O}\!\left(
\left(
1+\frac{C}{1-\gamma}(1+B)
\right)
\sqrt{
\frac{\log\bigl((I+1) |\mathcal{W}||\mathcal{V}|/\delta\bigr)}{n}
}
\right).
\]
\end{theorem}

\begin{proof}
Let
\[
\mathcal{E}
:=
\left\{
\sup_{w \in \mathcal{W},\, V \in \mathcal{V},\, \lambda \in (1 + \frac{1}{\varphi})\Delta^I}
\bigl|L(w,V,\lambda)-\widehat{L}(w,V,\lambda)\bigr|
\le \varepsilon_n
\right\}.
\]
By Lemma~\ref{lemma:concentration-state-value}, \(\mathbb{P}(\mathcal{E}) \ge 1-\delta\).
We work under the event \(\mathcal{E}\).

Let \(w^\ast = \mu^{\pi^\ast}/\mu_D \in \mathcal{W}\), which exists by
Assumption~\ref{assumption:miw-realizability}.
Since \(\pi^\ast\) is feasible, the corresponding occupancy measure \(\mu^{\pi^\ast}\)
satisfies
\[
\langle \mu^{\pi^\ast}, r_i \rangle = (1-\gamma)J_i(\pi^\ast) \ge \tau_i,
\qquad i=1,\dots,I.
\]
Moreover, since \(\mu^{\pi^\ast}\) satisfies the Bellman flow constraint, for every
\(V \in \mathcal{V}\) and every \(\lambda \in \Lambda_B\) we have
\[
L(w^\ast, V, \lambda)
=
(1-\gamma)J_0(\pi^\ast)
+
\sum_{i=1}^I \lambda_i\bigl((1-\gamma)J_i(\pi^\ast)-\tau_i\bigr)
\ge
(1-\gamma)J_0(\pi^\ast).
\]

Also, exactly as in the proof of Lemma~\ref{lemma:saddle-point}, for any
\(w \in \mathcal{W}\) and any \(\lambda \in \Lambda_B\),
\[
L\bigl(w, V_\lambda^{\pi(w)}, \lambda\bigr)
=
(1-\gamma)J_0(\pi(w))
+
\sum_{i=1}^I \lambda_i\bigl((1-\gamma)J_i(\pi(w))-\tau_i\bigr),
\]
where \(\pi(w)\) is the policy extracted from \(w\).

Now let \(\widehat\pi = \pi(\widehat w)\).
For any \(\lambda \in \Lambda_B\), since \((\widehat w,\widehat V,\widehat\lambda)\)
is a saddle point of \(\widehat L\), we obtain
\[
\begin{aligned}
(1-\gamma)J_0(\pi^\ast)
&\le L(w^\ast,\widehat V,\widehat\lambda) \\
&\le \widehat L(w^\ast,\widehat V,\widehat\lambda) + \varepsilon_n \\
&\le \widehat L(\widehat w,\widehat V,\widehat\lambda) + \varepsilon_n \\
&\le \widehat L(\widehat w, V_\lambda^{\widehat\pi}, \lambda) + \varepsilon_n \\
&\le L(\widehat w, V_\lambda^{\widehat\pi}, \lambda) + 2\varepsilon_n \\
&=
(1-\gamma)J_0(\widehat\pi)
+
\sum_{i=1}^I \lambda_i\bigl((1-\gamma)J_i(\widehat\pi)-\tau_i\bigr)
+
2\varepsilon_n .
\end{aligned}
\]
Hence, for all \(\lambda \in \Lambda_B\),
\[
(1-\gamma)J_0(\pi^\ast)
\le
(1-\gamma)J_0(\widehat\pi)
+
\sum_{i=1}^I \lambda_i\bigl((1-\gamma)J_i(\widehat\pi)-\tau_i\bigr)
+
2\varepsilon_n.
\]

Setting \(\lambda = 0\) gives
\[
(1-\gamma)J_0(\pi^\ast)
\le
(1-\gamma)J_0(\widehat\pi) + 2\varepsilon_n,
\]
so
\[
J_0(\widehat\pi) \ge J_0(\pi^\ast) - \frac{2\varepsilon_n}{1-\gamma}.
\]

Next, fix any \(i \in \{1,\dots,I\}\) and choose \(\lambda = Be_i \in \Lambda_B\).
Then
\[
(1-\gamma)J_0(\pi^\ast)
\le
(1-\gamma)J_0(\widehat\pi)
+
B\bigl((1-\gamma)J_i(\widehat\pi)-\tau_i\bigr)
+
2\varepsilon_n.
\]
Using the objective bound just proved,
\[
(1-\gamma)J_0(\pi^\ast) - (1-\gamma)J_0(\widehat\pi) \le 2\varepsilon_n,
\]
and therefore
\[
B\bigl((1-\gamma)J_i(\widehat\pi)-\tau_i\bigr)
\ge -4\varepsilon_n.
\]
Equivalently,
\[
(1-\gamma)J_i(\widehat\pi) \ge \tau_i - \frac{4\varepsilon_n}{B},
\qquad i=1,\dots,I.
\]
This completes the proof.
\end{proof}

\section{Unknown data distribution case} \label{appendix:unknown-data-distribution}

\subsection{Proof of Proposition~\ref{prop:negative2}}
For the proof, recall the reparameterized form of the Lagrangian:
$$
L(\bm\mu, \pi ; \bm{Q}, \bm\lambda)
= (1 - \gamma) \langle \bm{Q}_\pi, \bm{d}_0 \rangle + \langle \bm\mu, \bm{r}_{\bm\lambda} + \gamma \bm{P}\bm{Q}_\pi - \bm{Q} \rangle - \sum_{i=1}^I \lambda_i \tau_i.
$$

\begin{proof}
Consider a single-constraint CMDP with
$\mathcal S=\{s_0,l_1,r_1,l_2,r_2\}$, $\mathcal A=\{L,R\}$,
$\gamma=\frac12$ and $d_0=\delta_{s_0}$.
The transition kernel is
$P(l_1\mid s_0,L)=1$,
$P(r_1\mid s_0,R)=1$,
$P(l_2\mid r_1,L)=1$,
$P(r_2\mid r_1,R)=1$,
and
$P(l_1\mid l_1,a)=P(l_2\mid l_2,a)=P(r_2\mid r_2,a)=1$ for all $a\in\mathcal A$,
such that the states \(l_1,l_2,r_2\) are absorbing under both actions.
The reward functions are
$r_0(l_1,a)=1$, $r_0(l_2,a)=2$ for all $ a\in\mathcal A$
and \(r_0(s,a)=0\) for all other state-action pairs.
Set $r_1=r_0$ and $\tau_1=0$, such that every policy is feasible.

The optimal action-value function \(Q_0^\ast\) for reward \(r_0\) can be computed by solving the Bellman optimality equations:
$Q_0^\ast(l_1,a) = 2$, $Q_0^\ast(l_2,a)=4$, $Q_0^\ast(r_2,a)=0$ for all \(a\in\mathcal A\).
Also,
$Q_0^\ast(r_1,L)= 2$, $Q_0^\ast(r_1,R)=0$, $Q_0^\ast(s_0,L)= 1$, $Q_0^\ast(s_0,R)=1$.
$Q_0^\ast(s_0,L)=1$, $Q_0^\ast(l_1,L)=2$.

Now consider two policies $\pi^\ast$ and $\widehat\pi$ defined as follows:
$\pi^\ast(s)=L$ for all $s\in\mathcal S$ and
$\widehat\pi(s_0)=\widehat\pi(r_1)=R$, $\widehat\pi(l_1)=\widehat\pi(l_2)=\widehat\pi(r_2)=L$.
The policy \(\pi^\ast\) is optimal, because it moves from \(s_0\) to \(l_1\) and then stays there forever, achieving $J_0(\pi^\ast)=1$.
By contrast, \(\widehat\pi\) moves from \(s_0\) to \(r_1\), then from \(r_1\) to \(r_2\), and all rewards on this trajectory are zero. Hence
$J_0(\widehat\pi)= 0$. Therefore \(\widehat\pi\) is not optimal.

Let \(\mu^\ast=\mu^{\pi^\ast}\) be the occupancy measure of \(\pi^\ast\).
Since under \(\pi^\ast\) the process is at \((s_0,L)\) at time \(0\), and at \((l_1,L)\) at every time \(t\ge 1\),
$\mu^\ast(s_0,L)=(1-\gamma)=\frac12$, $\mu^\ast(l_1,L)=(1-\gamma)\sum_{t=1}^\infty \gamma^t=\frac12$,
and all other coordinates of \(\mu^\ast\) are zero.

Define $\mathcal U=\{\mu^\ast\}$, $\Pi=\{\pi^\ast,\widehat\pi\}$, and $\mathcal{Q} = \{ (1 + \alpha) Q_0^\ast : \alpha \geq 0 \}$.
Then, since $r_1 = r_0$, for every $\alpha \geq 0$, $Q_{r_0 + \alpha r_1}^\ast = (1 + \alpha) Q_0^\ast \in \mathcal{Q}$.
Hence, the realizability properties in item (i) and (ii) are satisfied.

We now show that \((\mu^\ast,\widehat\pi;Q_0^\ast,0)\) is a saddle point.
First, we verify primal optimality at \((Q_0^\ast,0)\).
Since \(\mu^\ast\) is supported only on \((s_0,L)\) and \((l_1,L)\), and since for both
\(\pi\in\Pi\) we have $Q_{0,\pi}^\ast(s_0)=1$ and $Q_{0,\pi}^\ast(l_1)=2$, the Bellman residuals on the support of \(\mu^\ast\) are
\[
r_0(s_0,L)+\gamma Q_{0,\pi}^\ast(l_1)-Q_0^\ast(s_0,L)
=
0+\frac12\cdot 2-1
=
0,
\]
and
\[
r_0(l_1,L)+\gamma Q_{0,\pi}^\ast(l_1)-Q_0^\ast(l_1,L)
=
1+\frac12\cdot 2-2
=
0.
\]
Hence, for every \(\pi\in\Pi\),
\[
L(\mu^\ast,\pi;Q_0^\ast,0)
=
(1-\gamma)Q_{0,\pi}^\ast(s_0)
=
\frac12.
\]
Because \(\mathcal U\) is a singleton, both \((\mu^\ast,\pi^\ast)\) and
\((\mu^\ast,\widehat\pi)\) maximize \(L(\mu,\pi;Q_0^\ast,0)\) over \(\mathcal U\times\Pi\).
In particular, \((\mu^\ast,\widehat\pi)\) is a primal maximizer.

Second, we verify dual optimality at \((\mu^\ast,\widehat\pi)\).
Fix arbitrary \(\alpha\ge 0\) and \(\lambda\ge 0\), and consider \(Q_\alpha^\ast=(1+\alpha)Q_0^\ast\).
Since \(\widehat\pi(l_1)=L\), we have
\[
Q_{\alpha,\widehat\pi}^\ast(s_0)=Q_\alpha^\ast(s_0,R)=1+\alpha,\qquad
Q_{\alpha,\widehat\pi}^\ast(l_1)=Q_\alpha^\ast(l_1,L)=2(1+\alpha).
\]
Again using that \(\mu^\ast\) is supported only on \((s_0,L)\) and \((l_1,L)\), we get
\[
(1+\lambda)r_0(s_0,L)+\gamma Q_{\alpha,\widehat\pi}^\ast(l_1)-Q_\alpha^\ast(s_0,L)
=
0+\frac12\cdot 2(1+\alpha)-(1+\alpha)
=
0,
\]
and
\[
(1+\lambda)r_0(l_1,L)+\gamma Q_{\alpha,\widehat\pi}^\ast(l_1)-Q_\alpha^\ast(l_1,L)
=
(1+\lambda)+\frac12\cdot 2(1+\alpha)-2(1+\alpha)
=
\lambda-\alpha.
\]
Therefore,
\[
L(\mu^\ast,\widehat\pi;Q_\alpha^\ast,\lambda)
=
(1-\gamma)Q_{\alpha,\widehat\pi}^\ast(s_0)
+
\mu^\ast(l_1,L)(\lambda-\alpha)
=
\frac12(1+\alpha)+\frac12(\lambda-\alpha)
=
\frac12(1+\lambda).
\]
Thus
\[
L(\mu^\ast,\widehat\pi;Q_\alpha^\ast,\lambda)\ge \frac12
=
L(\mu^\ast,\widehat\pi;Q_0^\ast,0)
\qquad
\text{for all } \alpha\ge 0,\ \lambda\ge 0.
\]
So \((Q_0^\ast,0)\) is a dual minimizer over \(\mathcal Q\times\mathbb R_+\).

Combining the previous two parts, \((\mu^\ast,\widehat\pi;Q_0^\ast,0)\) is a saddle point of the
reparameterized Lagrangian over \((\mathcal U\times\Pi)\times(\mathcal Q\times\mathbb R_+)\).
But \(\widehat\pi\) is not optimal, since \(J_0(\widehat\pi)=0<1=J_0(\pi^\ast)\).
\end{proof}

\subsection{Analysis of saddle points of lagrangian function} \label{sec:saddle-point-analysis}

Recall the Lagrangian function:
$$
\begin{aligned}
L(\bm\mu, \pi ; \bm{Q}, \bm\lambda)
&= (1 - \gamma) \langle \bm{Q}_\pi, \bm{d}_0 \rangle + \langle \bm\mu, \bm{r}_{\bm\lambda} + \gamma \bm{P}\bm{Q}_\pi - \bm{Q} \rangle - \sum_{i=1}^I \lambda_i \tau_i \\
&= \langle \bm{Q}_\pi, (1 - \gamma) \bm{d}_0  + \gamma \bm{P}^\top \bm\mu \rangle + \langle \bm\mu, \bm{r}_{\bm\lambda} - \bm{Q} \rangle - \sum_{i=1}^I \lambda_i \tau_i
\end{aligned}
$$
Change of variable $\bm{w} = \bm{\mu} / \bm{\mu}_D$ gives the equivalent form of the Lagrangian:
\[
L(\bm{w}, \pi ; \bm{Q}, \bm\lambda)
=
(1-\gamma)\langle \bm{Q}_\pi, \bm{d}_0 \rangle
+
\langle \bm{w}, \bm{r}_{\bm\lambda} + \gamma \bm{P}\bm{Q}_\pi - \bm{Q} \rangle_{\bm{\mu}_D}
-
\sum_{i=1}^I \lambda_i \tau_i,
\]
where \(\langle w, f\rangle_{\mu_D} = \sum_{s,a} w(s,a) \mu_D(s,a) f(s,a)\).
We first show useful identities of the Lagrangian function, which will be used in the proof of Lemma~\ref{lemma:key1-extreme}.
\begin{lemma}\label{lemma:langrangian-identity}
For every \(\pi\), \(Q\), and \(\lambda\), we have  
$$
L(w^\pi, \pi; \bm{Q}, \bm\lambda) = (1-\gamma)J_0(\pi) + \sum_{i=1}^I \lambda_i ((1-\gamma)J_i(\pi) - \tau_i),
$$
and
$$
L(w, \pi; \bm{Q}_{\bm\lambda}^\pi, \bm\lambda) = (1-\gamma)J_0(\pi) + \sum_{i=1}^I \lambda_i ((1-\gamma)J_i(\pi) - \tau_i),
$$
where $\bm{Q}_{\bm\lambda}^\pi$ is the value function of policy $\pi$ with respect to the reward $\bm{r}_{\bm\lambda}$.
\end{lemma}
\begin{proof}
To show the first identity, note that
$$
\begin{aligned}
L(w^\pi, \pi; \bm{Q}, \bm\lambda)
&= (1-\gamma)\langle \bm{Q}_\pi, \bm{d}_0 \rangle + \langle \bm{w}^\pi, \bm{r}_{\bm\lambda} + \gamma \bm{P}\bm{Q}_\pi - \bm{Q} \rangle_{\bm{\mu}_D} - \sum_{i=1}^I \lambda_i \tau_i \\
&= (1-\gamma)\langle \bm{Q}_\pi, \bm{d}_0 \rangle + \langle \bm\mu^\pi, \bm{r}_{\bm\lambda} + \gamma \bm{P}\bm{Q}_\pi - \bm{Q} \rangle - \sum_{i=1}^I \lambda_i \tau_i \\
&= \langle \bm{Q}_\pi, (1 - \gamma) \bm{d}_0 + \gamma \bm{P}^\top \bm\mu^\pi \rangle - \langle \bm\mu^\pi, \bm{Q} \rangle + \langle \bm\mu^\pi, \bm{r}_{\bm\lambda} \rangle - \sum_{i=1}^I \lambda_i \tau_i.
\end{aligned}
$$
By the Bellman flow equation, we have $(1 - \gamma) \bm{d}_0 + \gamma \bm{P}^\top \bm\mu^\pi = \bm{d}^\pi$ where $\bm{d}^\pi$ is the state occupancy measure of policy $\pi$.
Since
$$
\langle \bm{Q}_\pi, \bm{d}^\pi \rangle = \sum_{s, a} Q(s, a) \pi(a | s) d^\pi(s) = \sum_{s, a} Q(s, a) \mu^\pi(s, a) = \langle \bm\mu^\pi, \bm{Q} \rangle,
$$
we have
$$
L(w^\pi, \pi; \bm{Q}, \bm\lambda) = \langle \bm\mu^\pi, \bm{r}_{\bm\lambda} \rangle - \sum_{i=1}^I \lambda_i \tau_i = (1-\gamma)J_0(\pi) + \sum_{i=1}^I \lambda_i ((1-\gamma)J_i(\pi) - \tau_i),
$$
proving the first identity.
To show the second identity, note that if $\bm{Q}$ is a value function of some policy $\pi$ with respect to the reward $\bm{r}_{\bm\lambda}$, then the Bellman equation implies that $\bm{r}_{\bm\lambda} + \gamma \bm{P}\bm{Q}_\pi - \bm{Q} = 0$, which gives the second identity
\end{proof}

The following lemma shows objective bound and feasibility bound of a near-saddle point of \(L\) over a restricted domain.

\begin{lemma} \label{lemma:key1-extreme}
Assume Assumption~\ref{assumption:slater},\ref{assumption:all-policy-state-action-value-realizability}, \ref{assumption:miw-realizability} and \ref{assumption:boundedness}.
Suppose \((\widehat w,\widehat\pi;\widehat Q,\widehat\lambda)\) is a \(\xi\)-near saddle point of \(L\) over \((\mathcal W \times \Pi)\times(\mathcal Q \times \Lambda_B)\),
where \(B \coloneqq 1 + \frac{1}{\varphi}\) and \(\Lambda_B \coloneqq B\Delta^I\).
Then
\[
(1-\gamma)J_0(\widehat\pi) \ge (1-\gamma)J_0(\pi^\ast) - \xi
\qquad\text{and}\qquad
(1-\gamma)J_i(\widehat\pi) \ge \tau_i - \xi, \quad i=1,\dots,I,
\]
where $\pi^\ast \in \Pi$ is an optimal feasible policy with $w^{\pi^\ast} \in \mathcal W$.
\end{lemma}

\begin{proof}
Fix a comparator policy $\pi \in \Pi$ that is feasible and satisfies $w^\pi \in \mathcal W$.
Since $(\widehat w, \widehat \pi ; \widehat Q, \widehat \lambda)$ is a $\xi$-near saddle point,
applying the near-saddle inequality with $(w^{\pi^\ast},\pi^\ast)$ and $(Q_0^{\widehat\pi},0)$ gives
\[
L(w^{\pi^\ast},\pi^\ast;\widehat Q,\widehat\lambda)
\le
L(\widehat w,\widehat\pi;Q_0^{\widehat\pi},0)+\xi.
\]
Using the first identity in Lemma~\ref{lemma:langrangian-identity} and the fact that $\pi^\ast$ is feasible, we have
\[
L(w^{\pi^\ast},\pi^\ast;\widehat Q,\widehat\lambda)
=
(1-\gamma)J_0(\pi^\ast)
+
\sum_{i=1}^I \widehat\lambda_i((1-\gamma)J_i(\pi^\ast)-\tau_i)
\ge
(1-\gamma)J_0(\pi^\ast).
\]
On the other hand, the second identity in Lemma~\ref{lemma:langrangian-identity} implies
\[
L(\widehat w,\widehat\pi;Q_0^{\widehat\pi},0)=(1-\gamma)J_0(\widehat\pi).
\]
Hence, the saddle inequality gives
\[
(1-\gamma)J_0(\widehat\pi)\ge (1-\gamma)J_0(\pi^\ast)-\xi,
\]
which proves near-optimality of the objective.

For feasibility, let \((\mu^\ast,\pi^\ast;Q^\ast,\lambda^\ast)\) be a saddle point of the unrestricted population Lagrangian for the original CMDP.
By Lemma~\ref{lemma:dual-variable-bound}, \(\|\lambda^\ast\|_1 \le \varphi^{-1}\).
Define
\[
m \coloneqq \min_{i\in[I]} \bigl((1-\gamma)J_i(\widehat\pi)-\tau_i\bigr).
\]
Since \(w^{\widehat\pi}\) is admissible for the unrestricted primal problem, although it might not be in $\mathcal{W}$, we have
$$
\begin{aligned}
(1-\gamma)J_0(\pi^\ast)
&=
L(\mu^\ast,\pi^\ast;Q^\ast,\lambda^\ast) \\
&\ge
L(w^{\widehat\pi},\widehat\pi;Q^\ast,\lambda^\ast) \\
&=
(1-\gamma)J_0(\widehat\pi) + \sum_{i=1}^I \lambda_i^\ast ((1-\gamma)J_i(\widehat\pi)-\tau_i) \\
&\geq (1 - \gamma) J_0(\widehat\pi) + m \| \lambda^\ast \|_1,
\end{aligned}
$$
where the last equality is by the first identity in Lemma~\ref{lemma:langrangian-identity}, and the last inequality is by the definition of $m$ and the fact that $\lambda_i^\ast \geq 0$.
Rearranging, we get
\[
(1-\gamma)\bigl(J_0(\pi^\ast)-J_0(\widehat\pi)\bigr)\ge m\|\lambda^\ast\|_1.
\]

Now choose \(j \in \arg\min_{i\in[I]} \bigl((1-\gamma)J_i(\widehat\pi)-\tau_i\bigr)\), so that \(m=(1-\gamma)J_j(\widehat\pi)-\tau_j\).
Apply the near-saddle inequality with $(w^{\pi^\ast},\pi^\ast)$ and $(Q_{0,j}^{\widehat\pi}, B e_j)$:
\[
L(w^{\pi^\ast},\pi^\ast;\widehat Q,\widehat\lambda)
\le
L(\widehat w,\widehat\pi;Q_{0,j}^{\widehat\pi}, B e_j)+\xi.
\]
As before, feasibility of \(\pi^\ast\) implies
\[
L(w^{\pi^\ast},\pi^\ast;\widehat Q,\widehat\lambda)
= (1 - \gamma) J_0(\pi^\ast) + \sum_{i = 1}^I \widehat\lambda_i ((1 - \gamma) J_i(\pi^\ast) - \tau_i)
\ge (1-\gamma)J_0(\pi^\ast),
\]
where the first inequality uses the first identity in Lemma~\ref{lemma:langrangian-identity},
while the second identity in Lemma~\ref{lemma:langrangian-identity} gives 
\[
L(\widehat w,\widehat\pi;Q_{0,j}^{\widehat\pi}, B e_j)
=
(1-\gamma)J_0(\widehat\pi) + B\bigl((1-\gamma)J_j(\widehat\pi)-\tau_j\bigr)
=
(1-\gamma)J_0(\widehat\pi)+Bm.
\]
Combining and rearranging, we get
\[
(1-\gamma)\bigl(J_0(\pi^\ast)-J_0(\widehat\pi)\bigr)\le Bm+\xi.
\]
Combining with $(1 - \gamma)(J_0(\pi^\ast) - J_0(\widehat\pi)) \geq m \| \lambda^\ast \|_1$ obtained previously, we get
\[
m\|\lambda^\ast\|_1 \le Bm+\xi.
\]
If \(m\ge 0\), then \((1-\gamma)J_i(\widehat\pi)\ge \tau_i\) for all \(i\), so there is nothing to prove.
If \(m<0\), then
\[
m(B-\|\lambda^\ast\|_1)\ge -\xi.
\]
Since \(B=1+\varphi^{-1}\) and \(\|\lambda^\ast\|_1\le \varphi^{-1}\), we have \(B-\|\lambda^\ast\|_1 \ge 1\), and therefore \(m\ge -\xi\).
Thus
\[
(1-\gamma)J_i(\widehat\pi)\ge \tau_i-\xi,\qquad i=1,\dots,I.
\]
\end{proof}

\begin{proof}[Proof of Lemma~\ref{lemma:saddle-point-lagrangian-decomposition}]
This is a direct consequence of Lemma~\ref{lemma:key1-extreme} with $\xi = 0$.
\end{proof}

\begin{theorem} \label{theorem:pd-unconstrained-q}
Assume Assumptions~\ref{assumption:slater}, \ref{assumption:all-policy-state-action-value-realizability}, \ref{assumption:miw-realizability} and \ref{assumption:boundedness}.
Let $B \coloneqq 1 + \frac{1}{\varphi}$ where $\varphi$ is the Slater constant defined in Assumption~\ref{assumption:slater}.
Let $\pi^\ast \in \Pi$ be an optimal feasible policy assumed in Assumption~\ref{assumption:miw-realizability} such that \(w^{\pi^\ast} \in \mathcal{W}\).
If \((\widehat w,\widehat \pi;\widehat Q,\widehat \lambda)\) is a saddle point of
\[
\max_{w \in \mathcal W,\ \pi \in \Pi}\ \min_{Q \in \mathcal Q,\ \lambda \in \Lambda_B}\ \widehat L(w,\pi;Q,\lambda),
\]
where $\Lambda_B = B \Delta^I$, then with probability at least \(1-\delta\),
\[
(1-\gamma)J_0(\widehat\pi) \ge (1-\gamma)J_0(\pi^\ast) - \varepsilon_n,
\qquad
(1-\gamma)J_i(\widehat\pi) \ge \tau_i - \varepsilon_n,\quad i=1,\dots,I,
\]
where
\[
\varepsilon_n
=
\mathcal O\!\left(
\bigl(C^\ast(1+B+\frac{1}{1 - \gamma})+B\bigr)
\sqrt{\frac{\log((I+1) |\mathcal W|\,|\Pi|\,|\mathcal Q|/\delta)}{n}}
\right).
\]
\end{theorem}

We first prove the corresponding concentration bound for the Lagrangian estimate defined in \eqref{eqn:primal-dual-lagrangian-decomposition}:
$$
\widehat L(w,\pi;Q,\bm\lambda) \coloneqq (1-\gamma)Q(s_0,\pi) + \frac{1}{n}\sum_{j=1}^n w(s_j,a_j) ( r_{\bm\lambda}(s_j,a_j) + \gamma Q(s_j',\pi) - Q(s_j,a_j)) - \sum_{i = 1}^I \lambda_i \tau_i,
$$
where $r_{\bm\lambda}(s,a) = r_0(s,a) + \sum_{i=1}^I \lambda_i r_i(s,a)$.

\begin{lemma} \label{lemma:concentration}
Consider the empirical Lagrangian estimate
\[
\widehat L(w,\pi;Q,\bm\lambda) \coloneqq (1-\gamma)Q(s_0,\pi) + \frac{1}{n}\sum_{j=1}^n w(s_j,a_j) ( r_{\bm\lambda}(s_j, a_j) + \gamma Q(s_j',\pi) - Q(s_j,a_j)) -
\sum_{i=1}^I \lambda_i \tau_i.
\]
Given a function class  $\mathcal{W}$, $\Pi$ and $\mathcal{Q}$, with boundedness conditions $\|w\|_\infty \le C$ for all $w \in \mathcal{W}$, $\|Q\|_\infty \le \frac{1}{1-\gamma} $ for all $Q \in \mathcal{Q}$, we have with probability at least $1 -\delta$ that, for all $w \in \mathcal{W}$, $Q \in \mathcal{Q}$ and $\lambda \in B \Delta^I$,
$$
\vert L(w,\pi;Q,\lambda)-\widehat L(w,\pi;Q,\lambda) \vert \leq \varepsilon_n,
$$
where
\[
\varepsilon_n
=
\mathcal O\!\left(
\bigl(C(1+B+\frac{1}{1-\gamma})+B\bigr)
\sqrt{\frac{\log((I+1) |\mathcal W|\,|\Pi|\,|\mathcal Q|/\delta)}{n}}
\right).
\]
\end{lemma}

\begin{proof}
Fix $w \in \mathcal W$, $\pi \in \Pi$, $Q \in \mathcal Q$, and $\bm\lambda \in \{ 0, B \bm{e}_1, \dots, B \bm{e}_I \}$.
Since the term \((1-\gamma)Q(s_0,\pi)\) is deterministic, write
$$
Z_j \coloneqq w(s_j,a_j) ( r_{\bm\lambda}(s_j,a_j) + \gamma Q(s_j',\pi) - Q(s_j,a_j)) - \sum_{i = 1}^I \lambda_i \tau_i.
$$
Then
\[
L(w,\pi;Q,\lambda)-\widehat L(w,\pi;Q,\lambda)
=
\frac1n\sum_{j=1}^n (\mathbb E[Z_j]- Z_j).
\]
Using $r_0,r_i,\tau_i \in [0,1]$, $\|w\|_\infty \le C$, $\|\bm\lambda\|_1 \le B$, and $\|Q\|_\infty \le \frac{1}{1-\gamma}$, we have
\[
|Z_j|
\le
C(1+B+2\frac{1}{1-\gamma})+B
=
\mathcal O\!\bigl(C(1+B+\frac{1}{1-\gamma})+B\bigr).
\]
Hence, Hoeffding's inequality gives
\[
\bigl|L(w,\pi;Q,\lambda)-\widehat L(w,\pi;Q,\lambda)\bigr|
\le
\mathcal O\!\left(
\bigl(C(1+B+\frac{1}{1-\gamma})+B\bigr)\sqrt{\frac{\log(1/\delta)}{n}}
\right).
\]

Applying a union bound over \(\mathcal W \times \Pi \times \mathcal Q \times \{ 0, B \bm{e}_1, \dots, B \bm{e}_I\} \), we get that with probability at least \(1-\delta\),
\[
\bigl|L(w,\pi;Q,\bm\lambda)-\widehat L(w,\pi;Q,\bm\lambda)\bigr|
\le
\mathcal O\!\left(
\bigl(C(1+B+\frac{1}{1-\gamma})+B\bigr)
\sqrt{\frac{\log((I+1) |\mathcal W|\,|\Pi|\,|\mathcal Q|/\delta)}{n}}
\right).
\]
Now, for $\bm\lambda \in B \Delta^I$, note that it can be expressed as $\bm\lambda = B \sum_{i = 1}^I \alpha_i \bm{e}_i$ for some $\alpha_0, \alpha_1, \dots, \alpha_I \geq 0$ with $\sum_{i = 0}^I \alpha_i = 1$.
Since $L(w, \pi; Q, \lambda)$ and $\widehat{L}(w, \pi; Q, \lambda)$ are affine in $\lambda$, we have
$$
L(w, \pi; Q, \lambda) = \alpha_0 L(w, \pi; Q, 0) + \sum_{i = 1}^I \alpha_i L(w, \pi, Q, \bm{e}_i),
$$
and
$$
\widehat{L}(w, \pi; Q, \lambda) = \alpha_0 \widehat{L}(w, \pi; Q, 0) + \sum_{i = 1}^I \alpha_i \widehat{L}(w, \pi; Q, \bm{e}_i).
$$
The uniform concentration inequality we obtained for $\lambda \in \{0, B \bm{e}_1, \dots, B\bm{e}_I\}$ gives concentration inequality for all $\lambda \in B \Delta^I$.
This completes the proof.

\end{proof}

Now, we are ready to prove the theorem.

\begin{proof}[Proof of Theorem~\ref{theorem:pd-unconstrained-q}]
On the event of Lemma~\ref{lemma:concentration}, we have
\[
\sup_{w,\pi,Q,\lambda\in B \Delta^I} |L(w,\pi;Q,\lambda)-\widehat L(w,\pi;Q,\lambda)| \le \varepsilon_n.
\]
Since \((\widehat w,\widehat\pi;\widehat Q,\widehat\lambda)\) is a saddle point of \(\widehat L\), Lemma~\ref{lemma:approx-saddle-point} implies that it is a \(2\varepsilon_n\)-near saddle point of \(L\) over \((\mathcal W\times\Pi)\times(\mathcal Q\times B \Delta^I)\).
Applying Lemma~\ref{lemma:key1-extreme} with \(\xi=2\varepsilon_n\) gives
\[
(1-\gamma)J_0(\widehat\pi) \ge (1-\gamma)J_0(\pi^\ast) - 2\varepsilon_n,
\qquad
(1-\gamma)J_i(\widehat\pi) \ge \tau_i - 2\varepsilon_n,\quad i=1,\dots,I.
\]
Absorbing the factor \(2\) into the \(\mathcal O(\cdot)\) term completes the proof.
\end{proof}

\subsection{Analysis of primal-dual algorithm}

For the following theorem statement, we use the notation $\epsilon_\mathrm{opt}^{\pi}(T)$ to denote the average regret of the no-regret policy optimization oracle for policy $\pi$ after $T$ iterations, and $\epsilon_\mathrm{opt}^{w}(T)$ to denote the average regret of the no-regret linear optimization oracle for the $w$-player after $T$ iterations.
That is,
$$
\epsilon_\mathrm{opt}^{\pi}(T) = \max_{\pi' \in \Pi} \frac{1}{T} \sum_{t = 1}^T L(w_t, \pi'; Q_t, \lambda_t) - L(w_t, \pi_t; Q_t, \lambda_t),
$$
and
$$
\epsilon_\mathrm{opt}^{w}(T) = \max_{w' \in \mathcal{W}} \frac{1}{T} \sum_{t = 1}^T L(w', \pi_t; Q_t, \lambda_t) - L(w_t, \pi_t; Q_t, \lambda_t).
$$

\begin{theorem} \label{theorem:crl-pd}
Under Assumptions~\ref{assumption:slater}, \ref{assumption:all-policy-state-action-value-realizability}, \ref{assumption:miw-realizability} and \ref{assumption:boundedness}, running Algorithm~\ref{alg:pdocrl} with dual bound $B = 1 + \frac{1}{\varphi}$ and $T$ large enough so that $\epsilon_\mathrm{opt}^{\pi^\ast}(T) \le \varepsilon_n$ and $\frac{2(C^\ast)^2}{\sqrt{T}} + \frac{(1 + \frac{1}{\varphi} + \frac{2}{1 - \gamma})}{2\sqrt{T}} \le \varepsilon_n$,
and $\eta = 1 / \sqrt{T}$, we have with probability at least $1 - \delta$ that
$$
(1 - \gamma) J_0(\widehat\pi) \geq (1 - \gamma) J_0(\pi^\ast) - \varepsilon_n, \quad (1 - \gamma) J_i(\widehat\pi) \geq \tau_i - \varepsilon_n,~~i = 1, \dots, I.
$$
where $\varepsilon_n = \mathcal O( C^\ast(\frac{1}{\varphi} + \frac{1}{1 - \gamma}) \sqrt{\frac{\log(I |\mathcal W|\,|\Pi|\,|\mathcal Q|/\delta)}{n}})$.
\end{theorem}
\begin{proof}
Define $L(\pi, \lambda) = (1 - \gamma) J_0(\pi) + \sum_{i = 1}^I \lambda_i ((1 - \gamma) J_i(\pi) - \tau_i)$.
Let $\widehat\pi = \text{Unif}\{ \pi_1, \dots, \pi_T \}$ be the policy returned by the algorithm and let $\widehat\lambda = \frac{1}{T} \sum_{t = 1}^T \lambda_t$.
Since \(\widehat\pi\) samples one of \(\pi_1,\dots,\pi_T\) uniformly at the
beginning of the episode, we have
\[
\mathcal L(\widehat\pi,\lambda)=\frac1T\sum_{t=1}^T \mathcal L(\pi_t,\lambda).
\]
Also, by linearity in \(\lambda\),
\[
\mathcal L(\pi^\ast,\widehat\lambda)=\frac1T\sum_{t=1}^T \mathcal L(\pi^\ast,\lambda_t).
\]
By Lemma~\ref{lemma:key2}, it is enough to show that $L(\pi^\ast, \widehat\lambda) \leq L(\widehat\pi, \lambda) + \varepsilon_n$ for all $\lambda \in (1 + \frac{1}{\varphi}) \Delta^I$.
Consider the following decomposition:
\begin{align*}
L(\pi^\ast, \widehat\lambda) - L(\widehat\pi, \lambda)
&= \frac{1}{T} \sum_{t = 1}^T L(\pi^\ast, \lambda_t) - L(\pi_t, \lambda) \\
&= \frac{1}{T} \sum_{t = 1}^T L(w^\ast, \pi^\ast ; Q_t, \lambda_t) - L(w_t, \pi_t ; Q_\lambda^{\pi_t}, \lambda) \\
&= \frac{1}{T} \sum_{t = 1}^T L(w^\ast, \pi^\ast ; Q_t, \lambda_t) - L(w^\ast, \pi_t ; Q_t, \lambda_t) \\
&\hspace{15mm}+ \frac{1}{T} \sum_{t = 1}^T L(w^\ast, \pi_t ; Q_t, \lambda_t) - L(w_t, \pi_t ; Q_t, \lambda_t) \\
&\hspace{15mm}+ \frac{1}{T} \sum_{t = 1}^T L(w_t, \pi_t ; Q_t, \lambda_t) - L(w_t, \pi_t ; Q_\lambda^{\pi_t}, \lambda_t) \\
&\hspace{15mm}+ \frac{1}{T} \sum_{t = 1}^T L(w_t, \pi_t ; Q_\lambda^{\pi_t}, \lambda_t) - L(w_t, \pi_t ; Q_\lambda^{\pi_t}, \lambda),
\end{align*}
where the second equality is by Lemma~\ref{lemma:langrangian-identity}.

\paragraph{Regret of the \(\pi\)-player}
Using the uniform concentration bound on the Lagrangian estimate provided in Lemma~\ref{lemma:concentration}, the summand in the first term can be bounded by
$$
\begin{aligned}
L(w^\ast, \pi^\ast ; Q_t, \lambda_t) - L(w^\ast, \pi_t ; Q_t, \lambda_t)
&\leq
\widehat{L}(w^\ast, \pi^\ast ; Q_t, \lambda_t) - \widehat{L}(w^\ast, \pi_t ; Q_t, \lambda_t) + 2\varepsilon_n \\
&= (1 - \gamma) Q_t(s_0, \pi^\ast - \pi_t) + \gamma \langle \mu^\ast, PQ_t(\cdot, \pi^\ast - \pi_t) \rangle \\
&= \mathbb{E}_{\pi^\ast}[Q_t(s, \pi^\ast) - Q_t(s, \pi_t)],
\end{aligned}
$$
where the expectation is over $s \sim \sum_a \mu^{\pi^\ast}(\cdot, a)$.
The last equality is by the Bellman flow equation.
Since the $\pi$-player employs a no-regret policy optimization oracle (Definition~\ref{def:no-regret-oracle}), it follows that
$$
\frac{1}{T} \sum_{t = 1}^T L(w^\ast, \pi^\ast ; Q_t, \lambda_t) - L(w^\ast, \pi_t ; Q_t, \lambda_t) = \epsilon_{\mathrm{opt}}^{\pi^\ast}(T) \rightarrow 0,
$$
as $T \rightarrow \infty$, hence we can choose $T$ large enough so that the average regret of the $\pi$-player is at most $\varepsilon_n$.

\paragraph{Regret of the $w$-player}
To bound the second summation, note that
$$
\begin{aligned}
L(w^\ast, \pi_t ; Q_t, \lambda_t) - L(w_t, \pi_t ; Q_t, \lambda_t)
\leq
\widehat{L}(w^\ast, \pi_t ; Q_t, \lambda_t) - \widehat{L}(w_t, \pi_t ; Q_t, \lambda_t) + 2\varepsilon_n,
\end{aligned}
$$
and as discussed in Section~\ref{sec:w-player}, we can bound the average regret of the $w$-player by
$$
\frac{1}{T} \sum_{t = 1}^T \widehat{L}(w^\ast, \pi_t ; Q_t, \lambda_t) - \widehat{L}(w_t, \pi_t ; Q_t, \lambda_t)
\leq
\frac{2(C^\ast)^2}{\eta T} + \frac{\eta(1 + \frac{1}{\varphi} + \frac{2}{1 - \gamma} )}{2},
$$
and the choice $\eta = \frac{1}{\sqrt{T}}$ gives the bound $\frac{2(C^\ast)^2}{\sqrt{T}} + \frac{(1 + \frac{1}{\varphi} + \frac{2}{1 - \gamma})}{2\sqrt{T}}$.
Choosing $T$ large enough so that this bound is at most $\varepsilon_n$ completes the bound for the second term.

\paragraph{Regret of $Q$-player}
The summand in the third term can be bounded by
$$
\begin{aligned}
L(w_t, \pi_t ; Q_t, \lambda_t) - L(w_t, \pi_t ; Q_\lambda^{\pi_t}, \lambda_t)
\leq
\widehat{L}(w_t, \pi_t ; Q_t, \lambda_t) - \widehat{L}(w_t, \pi_t ; Q_\lambda^{\pi_t}, \lambda_t) + 2\varepsilon_n,
\end{aligned}
$$
which is bounded by $2 \varepsilon_n$ since the $Q$-player employs a greedy strategy, and the fact that $Q_0^{\pi_t}, Q_0^{\pi_t} + B Q_i^{\pi_t}$, $i = 1, \dots, I$ are members of $\mathcal{Q}$, and the fact that $\widehat{L}$ is affine in $Q$, as follows.
Let $\lambda_t = \sum_{i = 1}^I \alpha_i B e_i$ for some $\alpha_0, \alpha_1, \dots, \alpha_I \geq 0$ with $\sum_{i = 0}^I \alpha_i = 1$, then
$$
\begin{aligned}
&\widehat{L}(w_t, \pi_t ; Q_\lambda^{\pi_t}, \lambda_t) \\
&\quad=  \widehat{L}(w_t, \pi_t ; Q_0^{\pi_t} + \sum_{i = 1}^I \alpha_i B Q_i^{\pi_t}, \lambda_t) \\
&\quad=  \alpha_0 \widehat{L}(w_t, \pi_t ; Q_0^{\pi_t}, \lambda_t) + \sum_{i = 1}^I \alpha_i \widehat{L}(w_t, \pi_t ; Q_0^{\pi_t} + B Q_i^{\pi_t}, \lambda_t) \\
&\quad\geq \min \{ \widehat{L}(w_t, \pi_t ; Q_0^{\pi_t}, \lambda_t), \widehat{L}(w_t, \pi_t ; Q_0^{\pi_t} + B Q_1^{\pi_t}, \lambda_t), \dots, \widehat{L}(w_t, \pi_t ; Q_0^{\pi_t} + B Q_I^{\pi_t}, \lambda_t) \} \\
&\quad\geq \widehat{L}(w_t, \pi_t ; Q_t, \lambda_t).
\end{aligned}
$$

\paragraph{Regret of $\lambda$-player}
The summand in the fourth term can be bounded by
$$
\begin{aligned}
L(w_t, \pi_t ; Q_\lambda^{\pi_t}, \lambda_t) - L(w_t, \pi_t ; Q_\lambda^{\pi_t}, \lambda)
&\leq
\widehat{L}(w_t, \pi_t ; Q_\lambda^{\pi_t}, \lambda_t) - \widehat{L}(w_t, \pi_t ; Q_\lambda^{\pi_t}, \lambda) + 2\varepsilon_n \\
&= \widehat{L}(w_t, \pi_t ; Q_{t - 1}, \lambda_t) - \widehat{L}(w_t, \pi_t ; Q_{t - 1}, \lambda) + 2\varepsilon_n \\
&\leq 2 \varepsilon_n,
\end{aligned}
$$
where the last inequality holds since the $\lambda$-player employs a greedy strategy.
The equality is by the fact that the variables $Q$ and $\lambda$ are decoupled in the Lagrangian estimate $\widehat{L}(w, \pi ; Q, \lambda)$.

Combining the four bounds, we get
$$
L(\pi^\ast, \widehat\lambda) - L(\widehat\pi, \lambda) \leq \mathcal{O}(\varepsilon_n),
$$
and invoking Lemma~\ref{lemma:key2} completes the proof.
\end{proof}

\section{Analysis for offline unconstrained RL} \label{appendix:offline_unconstrained_Rl}

In this section, we provide a full analysis of the algorithm PDORL (Algorithm~\ref{alg:pdorl} for the offline unconstrained RL.
We first present a formal theorem of the guarantee.

\begin{theorem} \label{thm:pdorl-formal}
Assume $Q^\pi \in \mathcal{Q}$ for all $\pi \in \Pi$, and $w^{\pi_c} \in \mathcal{W}$ for some comparator policy $\pi_c \in \Pi$, and $\Vert w \Vert_\infty \leq C$ for all $w \in \mathcal{W}$ and $\Vert Q \Vert_\infty \leq 1 / (1 - \gamma)$ for all $Q \in \mathcal{Q}$.
Running Algorithm~\ref{alg:pdorl} with learning rate $\eta = 1 / \sqrt{T}$ and $T$ large enough so that $\epsilon_\mathrm{opt}^{\pi^\ast}(T) \le \varepsilon_n$, the output policy $\widehat\pi$ satisfies
$$
(1 - \gamma) J(\widehat\pi) \geq (1 - \gamma) J(\pi_c) - \varepsilon_n
$$
with probability at least $1 - \delta$, where $\varepsilon_n = \mathcal{O}(\frac{C}{1 - \gamma} \sqrt{\log (\vert \mathcal{W} \vert \vert \Pi \vert \vert \mathcal{Q} \vert / \delta) / n})$.
\end{theorem}

The algorithm runs a primal-dual algorithm to find a saddle point of the Lagrangian estimate:
$$
\widehat{L}(w, \pi ; Q)
= (1 - \gamma) Q(s_0, \pi) + \frac{1}{n} \sum_{i = 1}^n w(s_i, a_i)(r(s_i, a_i) + \gamma Q(s_i', \pi) - Q(s_i, a_i)),
$$
which is an unbiased estimator for the Lagrangian
$$
L(\bm\mu, \pi ; \bm{Q})
= (1 - \gamma) \langle \bm{Q}_\pi, \bm{d}_0 \rangle + \langle \bm\mu, \bm{r} + \gamma \bm{P}\bm{Q}_\pi - \bm{Q} \rangle.
$$
We first show a uniform concentration bound for the Lagrangian estimate.

\begin{lemma} \label{lemma:concentration2}
Consider a Lagrangian function estimate
$$
\widehat{L}(w, \pi ; Q)
= (1 - \gamma) Q(s_0, \pi) + \frac{1}{n} \sum_{i = 1}^n w(s_i, a_i)(r(s_i, a_i) + \gamma Q(s_i', \pi) - Q(s_i, a_i))
$$
Given a function class $\mathcal{W}$, $\Pi$ and $\mathcal{Q}$ for $w$, $\pi$ and $Q$, respectively, with boundedness condition $\Vert w \Vert_\infty \leq C$ for all $w \in \mathcal{W}$ and $\Vert Q \Vert_\infty \leq \frac{1}{1 - \gamma}$ for all $Q \in \mathcal{Q}$, we have with probability at least $1 - \delta$ that, for all $w \in \mathcal{W}$, $\pi \in \Pi$ and $Q \in \mathcal{Q}$,
$$
\vert L(w, \pi ; Q) - \widehat{L}(w, \pi ; Q) \vert \leq \varepsilon_n.
$$
where $\varepsilon_n = \mathcal{O}(\frac{C}{1 - \gamma} \sqrt{\log (\vert \mathcal{W} \vert \vert \Pi \vert \vert \mathcal{Q} \vert / \delta) / n})$.
\end{lemma}
\begin{proof}
The proof follows the proof of Lemma~\ref{lemma:concentration}, which shows a uniform concentration bound of the Lagrangian estimate for the constrained RL setting.
\end{proof}

\begin{lemma}\label{lemma:langrangian-identity-unconstrained}
For every \(\pi\), \(Q\), and \(\lambda\), we have  
$$
L(w^\pi, \pi; \bm{Q}) = (1-\gamma)J(\pi),
$$
and
$$
L(w, \pi; \bm{Q}^\pi) = (1-\gamma)J(\pi).
$$
\end{lemma}
\begin{proof}
It is a direct consequence of the identities under the constrained setting shown in Lemma~\ref{lemma:langrangian-identity}.
\end{proof}

We are now ready to show the theorem.

\begin{proof}[Proof of Theorem~\ref{thm:pdorl-formal}]
Let $\bar\pi = \text{Unif}(\pi_1, \dots, \pi_T)$ be the policy returned by Algorithm~\ref{alg:pdorl}.
It is a policy that first uniformly randomly draws from $\{\pi_1, \dots, \pi_T\}$, then follows the policy for the entire trajectory.
It follows that $J(\bar\pi) = \frac{1}{T} \sum_{t = 1}^T J(\pi_t)$, and
$$
\begin{aligned}
(1 - \gamma) J(\pi_c) - (1 - \gamma) J(\bar\pi)
&= \frac{1}{T} \sum_{t = 1}^T L(w^{\pi_c}, \pi_c ; Q_t) - L(w_t, \pi_t ; Q^{\pi_t}) \\
&= \frac{1}{T} \sum_{t = 1}^T L(w^{\pi_c}, \pi_c ; Q_t) - L(w^{\pi_c}, \pi_t ; Q_t) \\
&\hspace{20mm}+ \frac{1}{T} \sum_{t = 1}^T L(w^{\pi_c}, \pi_t ; Q_t) - L(w_t, \pi_t ; Q_t) \\
&\hspace{20mm}+ \frac{1}{T} \sum_{t = 1}^T L(w_t, \pi_t ; Q_t) - L(w_t, \pi_t ; Q^{\pi_t}),
\end{aligned}
$$
where the first equality is by the identities in Lemma~\ref{lemma:langrangian-identity}.
Note that the suboptimality of the policy $\bar\pi$ is decomposed into average regret terms of the three players.
We bound each of the regret terms as follows.

\paragraph{Regret of $\pi$-player}
The one-step regret of $\pi$-player can be written as
$$
\begin{aligned}
L(w^{\pi_c}, \pi_c ; Q_t) - L(w^{\pi_c}, \pi_t ; Q_t)
&= (1 - \gamma) \langle Q_t(\cdot, \pi_c - \pi_t), d_0 \rangle + \gamma \langle \mu^{\pi_c}, PQ_t(\cdot, \pi_c - \pi_t) \rangle \\
&= (1 - \gamma) \langle Q_t(\cdot, \pi_c - \pi_t), d_0 \rangle + \gamma \langle Q_t(\cdot, \pi_c - \pi_t), P^\top \mu_c \rangle \\
&= \langle Q_t(\cdot, \pi_c - \pi_t), (1 - \gamma) d_0 + \gamma P^\top \mu^{\pi_c} \rangle \\
&= \langle \bm{E}^\top \mu^{\pi_c}, Q_t(\cdot, \pi_c - \pi_t) \rangle \\
&= \sum_{s \in \mathcal{S}} \mu^{\pi_c}(s) Q_t(s, \pi_c - \pi_t) \\
&= \mathbb{E}_{\pi_c}[Q_t(\pi_c) - Q_t(\pi_t)],
\end{aligned}
$$
where the fourth equality is by the Bellman flow equation, and $\mathbb{E}_{\pi_c}$  is the expectation is over $s \sim \sum_a \mu^{\pi_c}(\cdot, a)$.
Since the $\pi$-player employs a no-regret policy optimization oracle (Definition~\ref{def:no-regret-oracle}), the regret is bounded by
$$
\frac{1}{T} \sum_{t = 1}^T L(w^{\pi_c}, \pi_c ; Q_t) - L(w_t, \pi_t ; Q^{\pi_t}) = \frac{1}{T} \sum_{t = 1}^T \mathbb{E}_{\pi_c}[Q_t(\pi_c) - Q_t(\pi_t)] \leq \epsilon_{\mathrm{opt}}^\pi(T) \rightarrow 0,
$$
as $T \rightarrow \infty$. Hence the regret can be made arbitrarily small by increasing $T$.

\paragraph{Regret of $w$-player}
To bound the second summation, note that
$$
\begin{aligned}
L(w^{\pi_c}, \pi_t ; Q_t) - L(w_t, \pi_t ; Q_t)
\leq
\widehat{L}(w^{\pi_c}, \pi_t ; Q_t) - \widehat{L}(w_t, \pi_t ; Q_t) + 2\varepsilon_n,
\end{aligned}
$$
and following the discussion in Section~\ref{sec:w-player}, we can bound the average regret of the $w$-player by
$$
\frac{1}{T} \sum_{t = 1}^T \widehat{L}(w^{\pi_c}, \pi_t ; Q_t, \lambda_t) - \widehat{L}(w_t, \pi_t ; Q_t, \lambda_t)
\leq
\frac{2 C^2}{\eta T} + \frac{\eta}{1 - \gamma},
$$
and the choice $\eta = \frac{1}{\sqrt{T}}$ gives the bound $\frac{2C^2}{\sqrt{T}} + \frac{1}{(1 - \gamma)\sqrt{T}}$.
Choosing $T$ large enough so that this bound is at most $\varepsilon_n$ completes the bound for the second term.

\paragraph{Regret of $Q$-player}
The one-step regret of $Q$-player can be bounded as
$$
\begin{aligned}
L(w_t, \pi_t ; Q_t) - L(w_t, \pi_t ; Q^{\pi_t})
&\leq
\widehat{L}(w_t, \pi_t ; Q_t) - \widehat{L}(w_t, \pi_t ; Q^{\pi_t}) + 2 \varepsilon_n \\
&\leq 2 \varepsilon_n
\end{aligned}
$$
where the first inequality is by the concentration inequality in Lemma~\ref{lemma:concentration2} and the second inequality is by the fact that $Q$-player chooses $Q_t \in \mathcal{Q}$ greedily and that $Q^{\pi_t} \in \mathcal{Q}$ by the all-policy realizability assumption.
Hence, the regret of $Q$-player can be bounded by
$$
\frac{1}{T} \sum_{t = 1}^T L(w_t, \pi_t ; Q_t) - L(w_t, \pi_t ; Q^{\pi_t}) \leq 2 \varepsilon_n.
$$
Combining the regret bounds completes the proof.
\end{proof}

\subsection{Comparison to previous work}

In this section, we compare our algorithm, PDORL (Algorithm~\ref{alg:pdorl}), with prior work on offline reinforcement learning with function approximation, where the performance objective is the infinite-horizon discounted sum of rewards.
See Table~\ref{table:comparison2} for the comparison.

\begin{table*}[h]
\caption{Comparison of algorithms for offline RL with general function approximation}
\label{table:comparison2}
\centering
\begin{tabular}{ccccccc}
 \toprule
 Algorithm & \makecell{Data \\ coverage} & \makecell{Oracle \\ efficient} & \makecell{Requires \\ $\mu_D$} & \makecell{Function \\ Approximation} & N \\
 \midrule
 FQI \parencite{munos2008finite} & \red{Full} & Yes & No & $\mathcal{T} \mathcal{Q} \subseteq \mathcal{Q}$ & $\epsilon^{-2}$ \\
 Minimax \parencite{xie2021bellman} & Any $\pi_c$ & \red{No} & No & $\forall \pi : \mathcal{T}^\pi \mathcal{Q} \subseteq \mathcal{Q}$ & \red{$\epsilon^{-5}$} \\
 Minimax \parencite{zanette2023realizability} & $\pi^\ast$ & \red{No} & No & $Q^\ast \in \mathcal{Q}$ & $\epsilon^{-2}$ \\
 PRO-RL \parencite{zhan2022offline} & $\pi^\ast$ & \red{No} & \red{Yes} & $w^\ast \in \mathcal{W}, V^\ast \in \mathcal{V}$ & \red{$\epsilon^{-6}$} \\
 A-Crab \parencite{zhu2023importance} & Any $\pi_c$ & \red{No} & No & $w^{\pi_c} \in \mathcal{W}$, $\forall \pi \in \Pi: Q^\pi \in \mathcal{Q}$ & $\epsilon^{-2}$ \\
 ATAC \parencite{cheng2022adversarially} & Any $\pi_c$ & \red{No} & No & $\forall \pi : \mathcal{T}^\pi \mathcal{Q} \subseteq \mathcal{Q}$, $\pi_c \in \Pi$ & \red{$\epsilon^{-3}$} \\
 CORAL \parencite{rashidinejad2022optimal} & $\pi^\ast$ & \red{No} & \red{Yes} & $w^\ast \in \mathcal{W}$, $V^\ast \in \mathcal{V}$, $\forall w \in \mathcal{W}: u_w \in \mathcal{U}$ & $\epsilon^{-2}$ \\
 \textbf{PDORL (Ours)} & Any $\pi_c$ & Yes & No & $\forall \pi \in \Pi: Q^\pi \in \mathcal{Q}$, $w^{\pi_c} \in \mathcal{W}$ & $\epsilon^{-2}$ \\
 \bottomrule
\end{tabular}
\end{table*}

Note that our algorithm is the only oracle-efficient algorithm that achieves sample complexity of $\epsilon^{-2}$ under partial data coverage that does not require the knowledge of the data generating distribution $\mu_D$.

\subsection{Projected gradient Ascent for $w$-player} \label{sec:w-player}

The $w$-player runs the projected gradient ascent algorithm on $\operatorname{co}(\mathcal{W}$.
The uniform concentration inequality for $\widehat{L}(w, \pi; Q, \lambda)$ over $w \in \mathcal{W}$ implies uniform concentration inequality over $w \in \operatorname(\mathcal{W})$ since $\widehat{L}$ is affine in $w$.
To analyze the guarantee of the projected gradient ascent algorithm employed by the $w$-player, we use this uniform concentration inequality, and the following standard regret bound for projected gradient ascent.

\begin{lemma} \label{lemma:pga}
Given $x_1 \in \mathcal{X} \subseteq \mathbb{R}^d$ where $\mathcal{X}$ is convex and $\eta > 0$, define  the sequences $x_2, \dots, x_{n + 1} \in \mathcal{X}$ and $h_1, \dots, h_n \in \mathbb{R}^d$ such that for $k = 1, \dots, n$,
$$
x_{k + 1} = \Pi_{\mathcal{X}} (x_k + \eta h_k)
$$
where $\Pi_{\mathcal{X}}(\cdot)$ is a projection onto $\mathcal{X}$.
Then, with the assumption that $\Vert h_k \Vert_2 \leq G$ for all $k = 1, \dots, n$, we have for any $x^\ast \in \mathcal{X}$:
$$
\sum_{k = 1}^n \langle x^\ast - x_k, h_k \rangle \leq \frac{\Vert x_1 - x^\ast \Vert_2^2}{2 \eta} + \frac{\eta n G^2}{2}.
$$
\end{lemma}
\begin{proof}
We start by bounding following term:
\begin{align*}
\Vert x_{k + 1} - x^\ast \Vert_2^2
&= \Vert \Pi_{\mathcal{X}} (x_k + \eta h_k) - \Pi_{\mathcal{X}}(x^\ast) \Vert_2^2 \\
&\leq \Vert (x_k + \eta h_k) - x^\ast \Vert_2^2 \\
&= \Vert x_k - x^\ast \Vert_2^2 + 2 \eta \langle x_k - x^\ast, h_k \rangle + \eta^2 \Vert h_k \Vert_2^2.
\end{align*}
Rearranging, we get
$$
2 \eta \langle x^\ast - x_k, h_k \rangle \leq \Vert x_k - x^\ast \Vert_2^2 - \Vert x_{k + 1} - x^\ast \Vert_2^2 + \eta^2 \Vert h_k \Vert_2^2.
$$
Summing over $k = 1, \dots, n$, we get
$$
2 \eta \sum_{k = 1}^n \langle x^\ast - x_k, h_k \rangle \leq \Vert x_1 - x^\ast \Vert_2^2 - \Vert x_{n + 1} - x^\ast \Vert_2^2 + \eta^2 \sum_{k = 1}^n \Vert h_k \Vert_2^2 \leq \Vert x_1 - x^\ast \Vert_2^2 + \eta^2 n G^2.
$$
Rearranging completes the proof.
\end{proof}

We now apply the lemma to the $w$-player. For the dataset
$\mathcal{D}=\{(s_i,a_i,s_i')\}_{i=1}^n$, define the empirical evaluation map
\[
E_{\mathcal{D}}(w)
\coloneqq
\bigl(w(s_1,a_1),\dots,w(s_n,a_n)\bigr)\in\mathbb{R}^n
\]
and the convex set
\[
\widetilde{\mathcal{W}}_{\mathcal{D}}
\coloneqq
\operatorname{co}\{E_{\mathcal{D}}(w): w\in\mathcal{W}\}
\subseteq \mathbb{R}^n.
\]

Since $\widehat L(w,\pi;Q,\lambda)$ depends on $w$ only through the
evaluation vector $E_{\mathcal D}(w)$, we may equivalently let the
$w$-player play directly in $\widetilde{\mathcal W}_{\mathcal D}$.

At round $t$, define the vector $g_t\in\mathbb{R}^n$ by
\[
[g_t]_i
\coloneqq
r_{\lambda_t}(s_i,a_i) + \gamma Q_t(s_i',\pi_t) - Q_t(s_i,a_i),
\qquad i=1,\dots,n.
\]
The $w$-player update is
\[
u_{t+1}
=
\Pi_{\widetilde{\mathcal{W}}_{\mathcal{D}}}(u_t + \eta g_t).
\]

The $w$-player aims to maximize the empirical Lagrangian, so for any comparator
$w^\ast \in \mathcal{W}$,
\[
\widehat{L}(w^\ast,\pi_t;Q_t,\lambda_t)
-
\widehat{L}(w_t,\pi_t;Q_t,\lambda_t)
=
\langle E_{\mathcal{D}}(w^\ast)-u_t,\; g_t\rangle.
\]
Applying Lemma~\ref{lemma:pga} with
$x_t=u_t$, $x^\ast=E_{\mathcal{D}}(w^\ast)$, and $h_t=g_t$, we obtain
$$
\begin{aligned}
\sum_{t=1}^T (\widehat{L}(w^\ast,\pi_t;Q_t,\lambda_t) - \widehat{L}(w_t,\pi_t;Q_t,\lambda_t))
&\leq
\frac{\|E_{\mathcal{D}}(w^\ast)-u_1\|^2}{2\eta} + \frac{\eta}{2}\sum_{t=1}^T \|g_t\|^2 \\
&\leq
\frac{2(C^\ast)^2}{\eta} + \frac{\eta T(B + \frac{2}{1 - \gamma} )}{2},
\end{aligned}
$$
where use the bounds in Assumption~\ref{assumption:boundedness} to bound $\Vert E_{\mathcal{D}}(w^\ast)-u_1\Vert^2$ and $\|g_t\|^2$.
Choosing $\eta = \frac{C^\ast}{(B + \frac{2}{1 - \gamma}) \sqrt{T}}$ gives further bounds $\frac{(B + \frac{2}{1 - \gamma})}{\sqrt{T}}$.

\section{Properties of saddle points} \label{sec:optimization}

\begin{lemma}[Lemma 31 in~\textcite{zhan2022offline}] \label{lemma:invariance-saddle-point}
Suppose $(x^\ast, y^\ast)$ is a saddle point of $f(x, y)$ over $\mathcal{X} \times \mathcal{Y}$, then for any $\mathcal{X}' \subseteq \mathcal{X}$ and $\mathcal{Y}' \subseteq \mathcal{Y}$, if $(x^\ast, y^\ast) \in \mathcal{X}' \times \mathcal{Y}'$, we have:
$$
\begin{aligned}
(x^\ast, y^\ast) &\in \argmin_{x \in \mathcal{X}'} \argmax_{y \in \mathcal{Y}'} f(x, y), \\
(x^\ast, y^\ast) &\in \argmax_{y \in \mathcal{Y}'} \argmin_{x \in \mathcal{X}'} f(x, y).
\end{aligned}
$$
\end{lemma}

\begin{proof}[Proof of Lemma~\ref{lemma:dual-variable-bound}]
Let $\pi^\ast$ be an optimal policy of the optimization problem $\mathcal{P}(\bm\tau)$.
Define the dual function $f(Q, \bm\lambda) = \max_{\mu \in \mathbb{R}_+^{\mathcal{S} \times \mathcal{A}}, \nu \in \mathbb{R}_+^{\mathcal{S} \times \mathcal{A}}} L(\mu, \nu ; Q, \lambda)$.
Let $(\bm{Q}^\ast, \bm\lambda^\ast) = \argmin_{\bm{Q} \in \mathbb{R}^{\mathcal{S} \times \mathcal{A}}, \bm\lambda \in \mathbb{R}_+^I} f(\bm{Q}, \bm\lambda)$.
Trivially, $\lambda_i^\ast \geq 0$ for all $i = 1, \dots, I$.
Also, by strong duality, we have $f(Q^\ast, \bm\lambda^\ast) = (1 - \gamma) J_0(\pi^\ast)$.

By the definition of the dual function, for any policy $\pi$, we have
\begin{align}
(1 - \gamma) J_0(\pi^\ast) = f(Q^\ast, \lambda^\ast)
&\geq L(\mu^{\pi}, \mu^{\pi} ; Q^\ast, \lambda^\ast) \notag \\
&= \langle \mu^{\pi}, r_0 \rangle + \langle \mu^{\pi}, R^\top \lambda^\ast \rangle - \langle \lambda^\ast, \tau \rangle \notag \\
&= (1 - \gamma) J_0(\pi) + (1 - \gamma) \sum_{i = 1}^I \lambda_i^\ast (J_i(\pi) - \frac{\tau_i}{1 - \gamma}). \label{eqn:suboptimality}
\end{align}

Let $\widehat{\pi}$ be a feasible policy with $J_i(\widehat{\pi}) \geq (\tau_i + \varphi) / (1 - \gamma)$ for all $i = 1, \dots, I$.
Such a policy exists by the assumption of this lemma.
Then, using the display above, we get
$$
\begin{aligned}
(1 - \gamma) J_0(\pi^\ast) 
&\geq (1 - \gamma) J_0(\widehat\pi) + (1 - \gamma) \sum_{i = 1}^I \lambda_i^\ast (J_i(\widehat\pi) - \frac{\tau_i}{1 - \gamma}) \\
&\geq (1 - \gamma) J_0(\widehat\pi) + \varphi \sum_{i = 1}^I \lambda_i^\ast \\
&= (1 - \gamma) J_0(\widehat\pi) + \varphi \Vert \lambda^\ast \Vert_1.
\end{aligned}
$$
Rearranging and using $1 / (1 - \gamma) \geq J_0(\pi^\ast) \geq J_0(\widehat\pi) \geq 0$ completes the proof:
$$
\Vert \bm\lambda^\ast \Vert_1 \leq \frac{(1 - \gamma) J_0(\pi^\ast) - (1 - \gamma) J_0(\widehat\pi)}{\varphi} \leq \frac{1}{\varphi}.
$$
\end{proof}

\begin{lemma} \label{lemma:approx-saddle-point}
If $(\widehat{x}, \widehat{y})$ is a saddle point of $\widehat{L}(\cdot, \cdot)$ over $\mathcal{X} \times \mathcal{Y}$ and $\vert L(x, y) - \widehat{L}(x, y) \vert \leq \xi / 2$ for all $(x, y) \in \mathcal{X} \times \mathcal{Y}$, then $(\widehat{x}, \widehat{y})$ is a $\xi$-near saddle point of $L(\cdot, \cdot)$.
\end{lemma}
\begin{proof}
Since $(\widehat{x}, \widehat{y}) $ is a saddle point of $\widehat{L}(\cdot, \cdot)$, we have for all $(x, y) \in \mathcal{X} \times \mathcal{A}$ that
$$
L(x, \widehat{y}) - L(\widehat{x}, y) \leq \widehat{L}(x, \widehat{y}) - \widehat{L}(\widehat{x}, y) + \xi \leq \xi,
$$
as required.
\end{proof}

\begin{lemma} \label{lemma:saddle-point-existence}
Consider the following Lagrangian function
$$
L(\bm\mu, \pi ; \bm{Q}, \bm\lambda)
= (1 - \gamma) \langle \bm{Q}_\pi, \bm{d}_0 \rangle + \langle \bm\mu, \bm{r}_{\bm\lambda} + \gamma \bm{P}\bm{Q}_\pi - \bm{Q} \rangle - \sum_{i=1}^I \lambda_i \tau_i,
$$
where we use the notation $Q_\pi(s) = \sum_{a'} \pi(a' | s) Q(s, a')$.
Then, the Lagrangian function has a saddle point over $(\mathcal{U} \times \Pi) \times (\mathcal{Q} \times \mathbb{R}_+^I)$, where $\mathcal{U}$, $\Pi$ and $\mathcal{Q}$ contain optimal variables $\mu^{\pi^\ast}$, $\pi^\ast$ and $Q_0^{\pi^\ast} + \sum_{i = 1}^I \lambda^\ast_i Q_i^{\pi^\ast}$ where $\pi^\ast$ is an optimal policy and $\lambda^\ast$ is an optimal dual variable.
\end{lemma}
\begin{proof}
By linear-programming strong duality, the decomposed Lagrangian~\eqref{eqn:lagrangian-decomposition} admits a saddle point over the unrestricted domains of $\mu$, $\nu$, $V$, $Q$, and $\lambda$.
Moreover, if $(\mu^\ast,\nu^\ast)$ is any optimal primal pair and $\pi^\ast$ is the policy extracted from $\nu^\ast$, then $\nu^\ast=\nu_{\mu^\ast,\pi^\ast}$, so the reparameterized Lagrangian~\eqref{eqn:lagrangian-decomposition-reparameterization} also has a saddle point over the unrestricted space of stationary policies.
Hence, when restricting the variables to $(\mathcal{U}, \Pi) \times (\mathcal{Q} \times \mathbb{R}_+^I)$ for some function classes $\mathcal{U}$, $\mathcal{Q}$, policy class $\Pi$, saddle points of the restricted Lagrangian are guaranteed to exist as long as the classes contain an optimal primal-dual pair, due to the invariance property of saddle points shown in Lemma~\ref{lemma:invariance-saddle-point}.
\end{proof}

\begin{lemma}\label{lemma:key2}
Consider a CMDP \eqref{eq:cmdp} that satisfies the Slater's condition with margin $\varphi$ (Assumption~\ref{assumption:slater}).
Define
\[
L(\pi,\lambda)
:=
(1-\gamma)J_0(\pi)
+
\sum_{i=1}^I \lambda_i\bigl((1-\gamma)J_i(\pi)-\tau_i\bigr).
\]
Let $\pi^\star$ be an optimal feasible policy. Suppose that
$\widehat\pi$ is a policy and $\widehat\lambda\in\Lambda_B \coloneqq \{ \lambda \in \mathbb{R}_+^I : \|\lambda\|_1 \le B \}$ satisfy
\[
L(\pi^\star,\widehat\lambda)\le L(\widehat\pi,\lambda)+\xi
\qquad\text{for all }\lambda\in\Lambda_B.
\]
Then, $\widehat\pi$ is a near-optimal, near-feasible solution:
\[
(1 - \gamma) J_0(\widehat\pi)\ge (1 - \gamma) J_0(\pi^\star)-\xi, \qquad (1-\gamma)J_i(\widehat\pi)\ge \tau_i-\xi.
\]
for every $i=1,\dots,I$,
\end{lemma}

\begin{proof}
Since $\pi^\ast$ is feasible, we have
\[
L(\pi^\ast, \widehat\lambda) = (1-\gamma) J_0(\pi^\ast) + \sum_{i=1}^I \widehat\lambda_i ((1 - \gamma) J_i(\pi^\ast) - \tau_i) \geq (1 - \gamma) J_0(\pi^\ast).
\]
Because $\bm{0} \in \Lambda_B$, the assumed inequality with $\bm\lambda=0$ gives
\[
(1-\gamma)J_0(\pi^\ast)
\le
L(\pi^\ast,\widehat\lambda)
\le
L(\widehat\pi,0)+\xi
=
(1-\gamma)J_0(\widehat\pi)+\xi,
\]
which shows the near optimality of $\widehat\pi$.

It remains to prove near feasibility. Consider the linear program formulation \eqref{eq:cmdp-lp} of the CMDP:
$$
\begin{aligned}
\max_{\bm\mu\ge 0}\quad & \langle \bm\mu, \bm{r}_0 \rangle \\
\text{subject to}\quad
& \bm{E}^\top \bm\mu = (1-\gamma)\bm{d}_0 + \gamma \bm{P}^\top \bm\mu, \\
& \langle \bm\mu, \bm{r}_i \rangle \ge \tau_i,\qquad i=1,\ldots,I.
\end{aligned}
$$
By strong duality for this linear program and Lemma~\ref{lemma:dual-variable-bound}, there exists a saddle point $(\bm\mu^\ast, \bm\lambda^\ast)$ of the corresponding Lagrangian $L(\bm\mu, \bm\lambda)$ where $\bm\lambda^\ast \in \Lambda_B$.
Also, it can be easily checked that $L(\pi, \bm\lambda) = L(\mu^\pi, \bm\lambda)$.
Hence,
$$
L(\pi, \bm\lambda^\ast) = L(\mu^\pi, \bm\lambda^\ast) \leq L(\bm\mu^\ast, \bm\lambda) = L(\pi^\ast, \bm\lambda),
$$
for every $\bm\lambda \geq \bm{0}$, where the inequality follows since $(\bm\mu^\ast, \bm\lambda^\ast)$ is a saddle point.
In particular, setting $\bm\lambda = \bm{0}$ and $\pi = \widehat\pi$, we get
\[
L(\widehat\pi,\lambda^\ast)\le (1-\gamma)J_0(\pi^\ast).
\]

Now define
$$
m \coloneqq \min_{i = 1, \dots, I} (1 - \gamma)J_i(\widehat\pi) - \tau_i,
$$
and choose $j\in\arg\min_{i} g_i(\widehat\pi)$, so that $g_j(\widehat\pi)=m$.
Then, the above inequality gives
\[
(1-\gamma)J_0(\pi^\ast)
\ge
L(\widehat\pi,\lambda^\ast)
=
(1-\gamma)J_0(\widehat\pi)
+
\sum_{i=1}^I \lambda_i^\ast ((1 - \gamma) J_i(\widehat\pi) - \tau_i)
\ge
(1-\gamma)J_0(\widehat\pi)+m\|\lambda^\ast\|_1.
\]
On the other hand, since $Be_j\in\Lambda_B$, the assumed inequality gives
\[
(1-\gamma)J_0(\pi^\ast)
\le
L(\pi^\ast,\widehat\lambda)
\le
L(\widehat\pi,Be_j)+\xi
=
(1-\gamma)J_0(\widehat\pi)+Bm+\xi.
\]
Combining the last two displays, we obtain
\[
m\|\lambda^\ast\|_1 \le Bm+\xi.
\]

If $m\ge 0$, then $(1 - \gamma) J_i(\widehat\pi) - \tau_i \geq 0$ for all $i$, so $\widehat\pi$ is feasible, and there is nothing to prove.
Suppose now that $m<0$. Then, rearranging the above inequality, we get
$$
m \geq \frac{-\xi}{B - \Vert \lambda^\ast \Vert_1} \geq -\xi
$$
where we use Lemma~\ref{lemma:dual-variable-bound} which gives $\Vert \lambda^\ast \Vert_1 \leq \frac{1}{\varphi} = B - 1$.
By the definition of $m$, this means
\[
(1 - \gamma) J_i(\widehat\pi) - \tau_i \geq -\xi
\qquad\text{for every }i=1,\dots,I,
\]
which shows the desired near-feasibility of $\widehat\pi$.
\end{proof}

\section{Experimental setup}\label{appendix:exp}
In this section, we explain the evaluation metric, the practical implementation of PDOCRL and the experimental setting of \ref{tab:bulletgym-result} and \ref{tab:ablation-bulletgym}. 

\subsection{Evaluation metrics} \label{appendix:metric}
In the metric setting, we follow a metric formula \cite{liu2023datasetsbenchmarksofflinesafe}. Denote $R^\pi$, $R_\text{max}(\mathcal{T})$ and $R_\text{min}(\mathcal{T})$ as empirical total reward evaluation of the policy, maximum, and minimum empirical total reward for task $\mathcal{T}$, respectively. Then the normalized reward is calculated by:
$$
R_{\text{normalized}} = \frac{R^\pi - R_{\text{min}}(\mathcal{T})}{R_{\text{max}}(\mathcal{T})-R_{\text{min}}(\mathcal{T})} 
$$

The normalized cost is calculated by dividing the total cost evaluation by the cost threshold. Denote $C^\pi$ and $C_\text{threshold}$ as the total cost evaluation and the cost threshold respectively. The normalized cost is calculated by:
$$
C_{\text{normalized}} = \frac{C^\pi}{C_\text{threshold}} 
$$

\subsection{Practical implementation}
\label{subsec:practical}

We now describe the practical instantiation of PDOCRL for deep RL experiments.
We parameterize the importance weight function $w$, the policy $\pi$, and the
state-action value function $Q$ as neural networks with parameters $\theta$,
$\phi$, and $\psi$, respectively. Since each benchmark task specifies a single
cost threshold $C$, the Lagrangian multiplier reduces to a scalar $\lambda$,
which we parameterize via an unconstrained scalar $\lambda_{\tau}$ (e.g.,
$\lambda = \operatorname{softplus}(\lambda_{\tau})$) to ensure
non-negativity.

The networks are trained by performing stochastic gradient descent--ascent on
the empirical Lagrangian derived from~\eqref{eqn:primal-dual-lagrangian-decomposition}:
$$
\begin{aligned}
&J(\theta, \phi, \psi, \lambda_{\tau}) \\
&\hspace{0.5cm}=
(1-\gamma)\,\overline{Q}_\psi(\pi_\phi)
\;+\;
\frac{1}{n}\sum_{j=1}^{n} w_\theta(s_j,a_j)
\bigl[r_\lambda(s_j,a_j)
+ \gamma\, Q_\psi(s_j',\pi_\phi) - Q_\psi(s_j,a_j)\bigr]
\;-\;\lambda\, C,
\end{aligned}
$$
where $\overline{Q}_\psi(\pi_\phi)
= \frac{1}{N}\sum_{i=1}^{N}Q_\psi(s_i,\pi_\phi)$ approximates the
initial-state term $(1-\gamma)\,Q(s_0,\pi)$ by averaging over all states in the
mini-batch, treating the empirical state distribution as a surrogate for the
initial state distribution.

We adopt three stabilization techniques common in offline RL.
First, to mitigate overestimation of $Q$, we maintain a target network
$\psi'$~\cite{58871} that is periodically synchronized with $\psi$.
Second, because the importance weight $w_\theta$ can also suffer from
overestimation, we clip its output to the interval $[0, 10]$.
Third, while Algorithm~\ref{alg:pdocrl} returns a policy sampled uniformly at
random over all iterates, in practice we return the last-iterate policy, which
we find to perform well empirically.

The resulting procedure is summarized in
Algorithm~\ref{alg:prac-pdocrl}.

\begin{algorithm}[t]
\caption{Practical PDOCRL}
\label{alg:prac-pdocrl}
\KwInput{Offline dataset
  $D = \{(s_0, s, a, r, c, s')_i\}_{i=1}^N$, learning rate $\eta$,
  Polyak averaging coefficient $\kappa$.}
Initialize parameters $\theta, \phi, \psi, \lambda_\tau$\;
Initialize target parameters $\psi' \leftarrow \psi$\;
\For{$t = 1, 2, \ldots$}{
    Sample a mini-batch from $D$\;
    \tcp{Ascent on primal variables}
    $\theta \leftarrow \theta + \eta\,\nabla_\theta J$\;
    $\phi \leftarrow \phi + \eta\,\nabla_\phi J$\;
    \tcp{Descent on dual variables}
    $\psi \leftarrow \psi - \eta\,\nabla_\psi J$\;
    $\lambda_\tau \leftarrow \lambda_\tau - \eta\,\nabla_{\lambda_\tau} J$\;
    \tcp{Soft target network update}
    $\psi' \leftarrow (1 - \kappa)\,\psi' + \kappa\,\psi$\;
}
\Return last-iterate policy $\pi_\phi$
\end{algorithm}

\subsection{Experimental setting} \label{Exp_setting}
In this section, we describe the experimental setup
for the results reported in Table~\ref{tab:bulletgym-result}
and Table~\ref{tab:ablation-bulletgym}.
The baseline results in Table~\ref{tab:bulletgym-result} are
taken from \cite{liu2023datasetsbenchmarksofflinesafe}.
For the baselines, all experiments use cost thresholds
$C \in \{20, 40, 80\}$ and three random seeds
$\in \{0, 10, 20\}$, yielding nine configurations per
environment. For PDOCRL, we use the same cost thresholds
and with seeds $\in \{0, 5, 10\}$.
For each configuration, the final checkpoint is evaluated
over 20 episodes for the baselines and 10 episodes for
PDOCRL, and the resulting reward and cost are normalized
following the protocol
of~\cite{liu2023datasetsbenchmarksofflinesafe}.
The values reported in the tables are averages of the
normalized reward and cost across all nine configurations.

All baseline algorithms adopt a common architecture of two
hidden layers of size $[256, 256]$ with ReLU activations,
are trained for $100{,}000$ gradient steps with a batch size
of 512, and use an actor learning rate of $10^{-4}$ and a
critic learning rate of $10^{-3}$. For Lagrangian-based
methods (BCQ-Lag, BEAR-Lag), the Lagrangian multiplier is
controlled by a PID controller with gains
$(K_P, K_I, K_D) = (0.1, 0.003, 0.001)$.
We refer the reader
to~\cite{liu2023datasetsbenchmarksofflinesafe} for
further algorithm-specific details.

For PDOCRL, the policy network~$\pi_\phi$,
the Q-networks~$Q_\psi$, and the importance weight
network~$w_\theta$ each use two hidden layers of size
$[256, 256]$ with ReLU activations, matching the baseline
architecture. We maintain two Q-networks and take the
minimum of their outputs to mitigate overestimation.
The target networks are updated via Polyak averaging with
coefficient $\tau = 0.005$. The actor, Q-network, and
importance weight networks are each optimized with a
learning rate of $3 \times 10^{-4}$, while the Lagrangian
parameter~$\lambda_\tau$ uses a learning rate of
$10^{-4}$. The reward signal is scaled by a factor of
$0.1$. We set the Slater margin to $\varphi = 1.0$ and
initialize the Lagrangian multiplier at
$\lambda_0 = 1.0$. Training runs for $100{,}000$ gradient
steps with a batch size of 512 and a discount factor
$\gamma = 0.99$. The hyperparameters of PDOCRL are
summarized in Table~\ref{tab:pdocrl-hyperparams}. For the ablation study, the baseline is identical to PDOCRL except for the policy extraction procedure. Lastly, All PDOCRL experiments were conducted on a local CPU-only
machine, without the use of GPU acceleration.
\begin{table}[t]
\centering
\caption{Hyperparameters for PDOCRL.}
\label{tab:pdocrl-hyperparams}
\begin{tabular}{lc}
\toprule
Hyperparameter & Value \\
\midrule
Hidden sizes (actor / critic / weight) & $[256, 256]$ \\
Number of Q-networks & 2 \\
Actor / Q / weight learning rate & $3 \times 10^{-4}$ \\
Lagrangian learning rate & $1 \times 10^{-4}$ \\
Batch size & 512 \\
Gradient steps & $100{,}000$ \\
Discount factor $\gamma$ & 0.99 \\
Target update coefficient $\tau$ & 0.005 \\
Reward scale & 0.1 \\
Cost scale & 1.0 \\
Slater margin $\varphi$ & 1.0 \\
Initial $\lambda$ & 1.0 \\
Evaluation episodes & 10 \\
Evaluation frequency & every $2{,}500$ steps \\
\bottomrule
\end{tabular}
\end{table}


\end{document}